\documentclass[11pt]{article}

\usepackage{xcolor}
\definecolor{brightmaroon}{rgb}{0.76, 0.13, 0.28}

\usepackage[
  shownumpages,               
  bgcolor={245,245,250},      
  braincolor={brightmaroon},    
  citingstyle=authoryear,     
  bibliostyle=plainnat,       
  bibfile=refs          
]{brainlab}

\usepackage{url}            
\usepackage{booktabs}       
\usepackage{amsfonts}       
\usepackage{nicefrac}       
\usepackage{microtype}      
\usepackage{xcolor}         

\usepackage{amsmath}
\usepackage{amssymb}
\usepackage{mathtools}
\usepackage{amsthm}
\usepackage{wrapfig}
\usepackage{nicefrac}
\allowdisplaybreaks

\usepackage{pifont}

\usepackage{algorithm}
\usepackage{algpseudocode}
\usepackage{multirow}

\usepackage{graphicx}
\usepackage{subfigure}
\usepackage{multirow}
\usepackage{colortbl}
\definecolor{bgcolor}{rgb}{0.8,1,1}
\definecolor{bgcolor2}{rgb}{0.8,1,0.8}
\usepackage{threeparttable}

\usepackage{tcolorbox}
\usepackage{pifont}
\definecolor{mydarkgreen}{RGB}{39,130,67}
\definecolor{mydarkred}{RGB}{192,47,25}

\usepackage[font=small]{caption}
\usepackage{subcaption}
\usepackage{enumitem}
\usepackage{wrapfig}

\renewcommand{\l}{\left}
\renewcommand{\r}{\right}
\newcommand{\E}{\mathbb{E}}

\usepackage{graphicx}
\usepackage{subfigure}
\usepackage{multirow}
\usepackage{colortbl}
\definecolor{bgcolor}{rgb}{0.8,1,1}
\definecolor{bgcolor2}{rgb}{0.8,1,0.8}
\usepackage{graphicx}
\usepackage{subcaption}
\usepackage{float} 

\usepackage[textsize=tiny]{todonotes}

\setbrainmeta{
  title={Zero-Order Optimization for LLM Fine-Tuning via Learnable Direction Sampling},
  authors={
      Valery Parfenov\textsuperscript{$\dagger$ 1,2},
      Grigoriy Evseev\textsuperscript{$\dagger$ 1},
      Andrey Veprikov\textsuperscript{1,3},
      Nikolay Bushkov\textsuperscript{4}, 
      Stanislav Moiseev\textsuperscript{4},
      Aleksandr Beznosikov\textsuperscript{1,3,4}
  },
  affiliations={
    \textsuperscript{$\dagger$}Equal contribution \\
    \textsuperscript{1}Basic Research of Artificial Intelligence Laboratory (BRAIn Lab) \\
    \textsuperscript{2}HSE University\\
    \textsuperscript{3}Federated Learning Problems Laboratory \\
    \textsuperscript{4}T-Technologies \\
    \textsuperscript{5}Innopolis University\\
  },
  abstract={
    Fine-tuning large pretrained language models (LLMs) is a cornerstone of modern NLP, yet its growing memory demands (driven by backpropagation and large optimizer States) limit deployment in resource-constrained settings. Zero-order (ZO) methods bypass backpropagation by estimating directional derivatives from forward evaluations, offering substantial memory savings. However, classical ZO estimators suffer from high variance and an adverse dependence on the parameter dimensionality $d$, which has constrained their use to low-dimensional problems. In this work, we propose a policy-driven ZO framework that treats the sampling distribution over perturbation directions as a learnable policy and updates it to reduce the variance of directional estimates. We develop a practical algorithm implementing this idea and provide a theoretical analysis, showing that learned sampling distributions improve the quality of gradient information and relax the explicit dependence on $d$ in convergence bounds. Empirically, we validate the approach on challenging LLM fine-tuning benchmarks, demonstrating substantially improved performance compared to standard ZO baselines. Our results suggest that adaptive direction sampling is a promising route to make ZO fine-tuning viable at scale. The source code is available at \url{https://github.com/brain-lab-research/zo_ldsd}.
  },
}

\begin{document}
\begin{mainpart}
\section{Introduction}

Fine-tuning pre-trained Large Language Models (LLMs) has become the dominant paradigm in contemporary natural language processing \cite{howard2018universal}, allowing large pretrained networks to be efficiently adapted to diverse downstream tasks with limited labelled data \citep{howard2018universal, zhang2019dialogpt, lester2021power}. The fine-tuning objective is commonly cast as an unconstrained optimization problem
\begin{equation}
\label{eq:main_problem}
\min_{x \in \mathbb{R}^d} f(x),
\end{equation}
where $x$ denotes the model parameters, $d$ is the dimension of $x$ (number of trainable parameters) and $f(x)$ is the loss.

Algorithms that rely on access to gradients remain the prevailing tools for solving the problem \eqref{eq:main_problem}. Classical first-order optimizers such as Stochastic Gradient Descent \cite{robbins1951stochastic,amari1993backpropagation} and Adam \cite{kingma2014adam} are widely used because they integrate naturally with backpropagation and offer strong empirical performance \cite{semenov2025benchmarking,wen2025fantastic}.

Despite these advances in stability and the practical convenience afforded by backpropagation, its memory footprint poses a growing practical barrier as model sizes increase \cite{rajbhandari2020zero}. Storing activations and optimizer States for gradient computation can dominate the memory budget: reported measurements indicate that gradient computation for models at the tens-of-billions parameter scale may require approximately $12$ times more memory than inference \citep{zhu2023efficient}. This constraint restricts accessible fine-tuning workflows on resource-limited hardware (for example, edge devices and consumer GPUs) and complicates large-scale distributed training \citep{zhu2023pockengine, gao2024enabling, liao2024lohan}.

To mitigate memory pressure, a range of strategies has been explored. One family comprises coordinate-wise or block-coordinate schemes which can reduce the number of gradient parameters from $d$ to $1$ \cite{richtarik2014iteration}. However, this comes at the cost of an asymptotic slowdown in convergence by a factor of $d$ \cite{li2018faster} (see Section \ref{sec:algo} for details). One of the most prominent representatives of coordinate methods is BAdam \cite{luo2024badam}, which reduces memory consumption by performing block-wise adaptive updates. While this lowers the storage requirements for optimizer States, as is typical for coordinate-wise optimization methods, it introduces a trade-off in the form of slower convergence, with the iteration complexity increasing proportionally to the number of blocks and, in the worst case, to the dimension $d$.

Motivated by these memory–convergence trade-offs in first-order schemes, Zero-order (ZO) methods offer a principled alternative. Rather than computing exact gradients via backpropagation, ZO approaches reconstruct directional derivatives (or gradient projections) from function evaluations by perturbing the objective along randomized directions. As a consequence, ZO methods require substantially less memory than coordinate-wise first-order methods such as BAdam \cite{malladi2023fine}, since they avoid storing intermediate activations and optimizer States associated with gradient computation. By estimating descent directions using only such finite-difference information, ZO methods eliminate the need for backward passes and the associated storage of intermediate activations. This black-box viewpoint decouples optimization from the internal implementation of the model and enables fine-tuning with substantially smaller memory budgets. Recent work has demonstrated that applying classical ZO updates to LLM fine-tuning can reduce memory consumption by several times compared to standard first-order pipelines \citep{zo_bench}. 
The first recognition of this idea in the context of LLM fine-tuning came with the \texttt{MeZO} \citep{malladi2023fine} method, where the authors employed a \texttt{ZO SGD} \citep{ghadimi2013stochastic} update based on the following gradient estimator:
\begin{eqnarray}\label{eq:zo_oracle}
    v \cdot \langle \nabla f(x), v \rangle \approx v \cdot \frac{f(x+\tau v)-f(x-\tau v)}{2\tau},
\end{eqnarray}
where random direction $v \in \mathbb{R}^d$ typically sampled from $\mathcal{N}(0,1)$ \cite{nesterov2017random}, and $\tau > 0$ is the smoothing parameter. 

At a high level, ZO-SGD as well as other ZO methods share a similar design philosophy with coordinate approaches, however, rely exclusively on forward-only information to approximate gradients. Estimators of the form \eqref{eq:zo_oracle} constitute a standard design choice for ZO schemes and are widely used in the modern literature \cite{chen2017zoo,malladi2023fine,petrov2025leveraging}. Moreover, in the limit $\tau\rightarrow 0$ the method recovers coordinate SGD.

Therefore ZO approaches suffer from inherent limitation: gradient estimates based on random directional derivatives exhibit high variance, similar to coordinate methods. Moreover, as the dimensionality $d$ increases, the probability that a random direction forms a large angle with the true gradient grows, leading to a further deterioration in the accuracy of the estimated descent direction. In essence, ZO methods obtain information only along a limited set of directions, behaving analogously to coordinate- or block-wise schemes. Consequently, they inherit a principal drawback: the effective convergence rate degrades by a factor of $d$ \cite{zhan2024unlocking}. 
Consequently, until recently, the use of ZO methods were mostly restricted to low-dimensional settings \cite{pmlr-v84-wang18e}. However, in light of modern applications in LLM fine-tuning, where the dimensionality $d$ reaches $10^{10}$ \cite{seung2025low}, this limitation has become even more pronounced. 


This work aims to mitigate the restrictive role of variance in ZO schemes. We view the ZO optimization process through the lens of Reinforcement Learning (RL), interpreting the distribution mean of the direction $v$ as a policy. By introducing a plug-and-play extension to the ZO framework, we enable a memory-efficient exploitation of information from previous iterations to improve the quality of gradient approximation. Our goal is both to provide theoretical insight and to empirically validate the proposed approach in real-world application settings.

Our contributions are as follow:
\begin{itemize}
    \item We introduce a novel ZO method, based on optimization of distribution to avoid uninformative directions.

    \item We provide theoretical guarantees for the proposed method under the smooth non-convex setup. Novel design leads to non-standard challenges, theoretical analysis and bounds. The resulting convergence guarantees exhibit no dependence on the dimension $d$.

    \item We empirically evaluate the proposed ZO method on challenging LLM fine-tuning benchmarks, demonstrating their effectiveness and practical relevance.
\end{itemize}

\section{Related Work}

\subsection{Memory efficient approaches}

We note that, beyond ZO approaches, the memory-efficient optimization literature includes several orthogonal techniques, such as gradient quantization to lower precision \citep{novikov2023few}, compression of optimizer statistics \citep{dettmers20218, li2023memory}, forward gradient learning \citep{baydin2022gradients, silver2021learning}, greedy layer-wise strategy \citep{nokland2019training}, synthetic gradients \citep{jaderberg2017decoupled} and others \citep{lv2024full, boopathy2022train, hinton2022forward, singhal2023guess}.

Nevertheless, ZO methods have emerged as the most effective solutions for memory-efficient LLM fine-tuning \cite{malladi2023fine}. The most apparent distinction among existing ZO approaches lies in the choice of the sampling distribution for the random direction $v$.
Commonly used options include the uniform distribution over the unit sphere $RS(1)^d_{\|\cdot\|}$ \citep{flaxman2004online}, the standard multivariate Gaussian distribution $\mathcal{N}(0,I)$ \citep{ghadimi2013stochastic,nesterov2017random}, and a uniform discrete distribution over the vectors of the canonical (one-hot) basis \citep{duchi2015optimal}. These choices yield unbiased gradient estimators of the smoothed function $f$, which form the foundation of the theoretical analyses of the ZO methods.

Recent works have explored complementary strategies for making ZO optimization more practical and memory-efficient. ZO-AdaMM \cite{chen2019zo} generalizes adaptive-moment methods to the gradient-free regime, combining momentum with per-coordinate adaptive step-sizes using only function evaluations. JAGUAR \cite{veprikov2024new,petrov2025leveraging} implements a momentum-focused ZO design and related matrix-aware extensions, emphasizing minimal per-iteration oracle cost and memory-efficient updates for large-model fine-tuning.

\subsection{ZO optimization with adaptive direction sampling}

Recent advances have aimed to inject richer curvature information into ZO procedures to improve stability and convergence. One of the first approaches in this direction was proposed in HiZOO \cite{zhao2024second}: it augments zero-order gradient (\ref{eq:zo_oracle}) with precondition matrix, using a Hessian-informed strategy to mitigate heterogeneous curvature across parameters.

Another recent perspective casts ZO finite-difference procedures as instances of policy optimization. In the work \cite{qiu2025zeroth}, authors show that standard ZO finite-difference estimators are mathematically equivalent to a single-step policy optimization objective and, equivalently, to a REINForCE-style gradient estimator \cite{williams1992simple} with a particular baseline. Building on this insight, they propose ZoAR method, which incorporates policy-optimization–inspired variance reduction through an averaged baseline and query reuse. 

A further contribution is presented in \cite{seung2025low}, which introduces LOREN: a curvature-aware ZO method that formulates gradient preconditioning as adaptive estimation of an anisotropic perturbation distribution, employs a low-rank block-diagonal preconditioner within a natural evolution strategies framework, and uses a REINForCE leave-one-out estimator to improve convergence and search efficiency.

In contrast to these approaches, our method focuses on adaptively learning the mean of the sampling distribution rather than estimating or approximating second-order curvature information. Unlike HiZOO and LOREN, which rely on explicit or implicit Hessian-related preconditioning and therefore introduce additional structural assumptions and memory overhead, our framework remains lightweight and preserves the simplicity of first-order ZO updates. At the same time, compared to policy-optimization–based ZO methods such as ZoAR, which primarily target variance reduction through baselines and query reuse, we directly bias the sampling distribution toward empirically useful directions by learning the distribution mean. This allows us to improve directional alignment without modifying the underlying ZO estimator or increasing algorithmic complexity, making the proposed framework a modular plug-in that can be seamlessly combined with a wide range of existing ZO optimizers.


\section{Algorithm and Theoretical Analysis under Directional Setup}\label{sec:algo}

\subsection{Notation and Assumptions}
For any vector $x \in \mathbb{R}^d$ we denote by $\|x\| = \sqrt{\sum_{i=1}^d x_i^2}$ the Euclidean $l_2$ norm. For any non-zero $x \neq 0$ we define its normalized version as $\overline{x} = \frac{x}{\|x\|}$.
We conduct the theoretical analysis in non-convex setting under the standard assumptions \citep{bertsekas1997nonlinear, ghadimi2013stochastic, richtarik2021ef21, li2021page}. 
\begin{assumption}\label{ass:lip}
    Let $f$ be $L$-smooth, i.e., for all $x_1,x_2 \in \mathbb{R}^d$ it holds that:
    $
        \|\nabla f(x_1) - \nabla f(x_2)\| \leq L\|x_1 - x_2\|.
    $
\end{assumption}

\begin{assumption}\label{ass:minimizer}
    Let there exist a minimizer of $f$ i.e., an $x^*$ that 
    $
        f(x) \geq f(x^*)
    $
    holds for all $x \in \mathbb{R}^d$.
\end{assumption}

\subsection{Motivation}

We consider solving the  problem \eqref{eq:main_problem} using optimization along random directions, i.e., via a Directional Gradient Descent (DGD) \citep{silver2021learning}:
\begin{eqnarray}\label{eq:cgd_iteration}
    x^{t+1} = x^t - \gamma_x^t \cdot \overline{v^t}\langle \overline{v^t}, \nabla f(x^t)\rangle,
\end{eqnarray}
where the sampling distribution of $v^t$ is arbitrary.
We note that ZO oracle \eqref{eq:zo_oracle} reduces to DGD estimator under the $\tau \to 0$ condition. It makes the role of the direction sampling distribution identical in both formulations and allows us to focus on the DGD setup in the theoretical analysis.

Intuitively, the quality of the estimator improves as the sampled direction becomes better aligned with the gradient, and it attains its maximum accuracy in the case of a deterministic choice of $v^t$ collinear with $\nabla f (x^t)$.
To formalize this intuition, we introduce the notion of \textit{gradient alignment}:
\begin{eqnarray}\label{eq:C_def}
    C^t := \left\langle\overline{v^t},\overline{\nabla{f(x^t)}}\right\rangle^2,
\end{eqnarray}
and use it in Lemma~\ref{lem:descent}, inspired by \citep{cheng2021convergence}.

\begin{lemma}\label{lem:descent}
Under Assumptions \ref{ass:lip} and \ref{ass:minimizer} for iterations of DGD \eqref{eq:cgd_iteration}
the following inequality holds:
    \begin{eqnarray*}
    &&\sum_{t=0}^{T-1} \frac{1}{T} \E\l[ \l(\gamma_x^t - \frac{L {\gamma_x^t}^2}{2} \r)\E\l[C^{t}|\mathcal{F}^{t-1}\r]\|\nabla f(x^{t})\|^2 \r] \notag \leq \frac{f(x^0)-f(x^{*})}{T},
\end{eqnarray*}
where $\mathcal{F}^t$ is a $\sigma$-algebra generated by $\{x^0, x^1, \dots, x^t\}$.
\end{lemma}




One of its key features is that Lemma \ref{lem:descent} applies to an arbitrary sampling distribution. Crucially for our analysis, it does not rely on the unbiasedness of the oracle, in contrast to commonly used approaches \citep{nesterov2017random}. This significantly broadens the range of settings that can be analysed within this framework.

The simplest choice of the sampling distribution is $v^t \sim \mathcal{N}(\mu^t \equiv 0,\varepsilon^2I)$: a Gaussian distribution with zero mean, which yields an unbiased estimator in the DGD setting. After normalizing the sampled vector, this yields a uniform distribution over the unit sphere $\overline{v^t} {\sim} RS(1)^d_{\|\cdot\|}$. Notably, by leveraging Lemma~\ref{lem:descent}, we recover the convergence rate for DGD with $v^t \overset{i.i.d.}{\sim} \mathcal{N}(0,\varepsilon^2I)$ established in \citep{nesterov2017random}, as Stated in Corollary~\ref{lem:ZO_SGD}.

\begin{corollary}\label{lem:ZO_SGD}
    Let in the setting of Lemma \ref{lem:descent} $v^t \overset{i.i.d.}{\sim} \mathcal{N}(\mu^t \equiv 0,\varepsilon^2I)$ and $\gamma^t_x \leq \frac{1}{2L}$. Then the following bound holds:
    \begin{eqnarray*}
        \E \| \nabla f(\overline{x}^T)\|^2 \leq 
        \frac{2d}{\gamma_x^t T}(f(x^0) - f(x^*)),
    \end{eqnarray*}
    where $\overline{x}^T$ is chosen uniformly from $\{x^0,x^1,\dots,x^{t-1}\}$.
\end{corollary}
The proof of Corollary \ref{lem:ZO_SGD} is provided in Appendix \ref{sec:mis_proofs} and directly reveals that the $d$ factor arises from $\E\l[C^{t}|\mathcal{F}^{t-1}\r] \sim 1/d$ for $v^t \overset{i.i.d.}{\sim} \mathcal{N}(0,\varepsilon^2I)$. 

The above intuition shows that, on the one hand, Lemma \ref{lem:descent} indeed captures the variance effect induced by the randomness of the direction, and on the other hand, its application yields tight bounds.
This makes expected gradient alignment $\E\l[C^{t}|\mathcal{F}^{t-1}\r]$ a key concept in the analysis of the DGD setup. 
To obtain theoretical acceleration, it is necessary to maintain this quantity at an $\mathcal{O}(1)$ level, in contrast to the baseline $\mathcal{O}(1/d)$ regime. To this end, we propose to leverage information about the objective function obtained during previous iterations by learning the parameters of the sampling distribution and introducing control over them to increase the value of $\E\l[C^{t}|\mathcal{F}^{t-1}\r]$.

Using a Gaussian distribution gives rise to two design choices: adapting the mean vector $\mu^t$ or the covariance matrix $\Sigma^t$. We focus on learning $\mu^t$ and fix $\Sigma^t \equiv \varepsilon^2 I$, since in the context of memory-efficient optimization storing a full covariance matrix is infeasible, whereas learning the mean vector incurs only an additional $\mathcal{O}(d)$ memory overhead, which is acceptable in the ZO setting.


\subsection{Theoretical Algorithm}

The above mentioned intuition, along with Lemma \ref{lem:descent}, motivates the design of Algorithm \ref{alg:theor1}: \texttt{LDSD} (\textbf{L}earnable \textbf{D}irection \textbf{S}ampling \textbf{D}escent). 
\begin{algorithm}{\textsc{LDSD}}\label{alg:theor1}
\begin{algorithmic}[1]
\State \textbf{Input:} Initial points $x^0$ and $\mu^0$, number of iterations $T$, variance scale $\varepsilon$.
\For{$t = 0,1,2,\dots, T-1$}
    \State Sample $v^t \sim \mathcal{N}(\mu^t, \varepsilon^2I)$ and set $\overline{v^t} = v^t / \| v^t \|$ \label{line:sampling}
    \State $g_x^t = \overline{v^t} \l\langle\overline{v^t},{\nabla f(x^t)}\r\rangle$\label{line:g_x}
    \State $g_{\mu}^t = \nabla_{\mu} \E\left[C^{t}|\mathcal{F}^{t-1}\right]\Big|_{\mu = \mu^{t}}$ \label{line:g_v}
    \State $\mu^{t+1} = \mu^{t} + \gamma_\mu^t g_{\mu}^t$ \label{line:mu_update}
    \State $x^{t+1} = x^t - \gamma_x^t g_x^t$ \label{line:x_update_th}
\EndFor
\end{algorithmic}
\end{algorithm}
Lines \ref{line:sampling}, \ref{line:g_x}, \ref{line:x_update_th} of Algorithm \ref{alg:theor1} reflect the classical use of DGD estimator.
In contrast, Lines \ref{line:g_v} and \ref{line:mu_update} incorporate the proposed mechanism for learning the sampling direction.
Specifically, to attain the values $\E\l[ C^{t}\r | \mathcal{F}^{t-1}]$ of expected gradient alignment as large as possible, we adopt a reinforcement learning perspective by interpreting the mean vector $\mu^t$ as a policy and $\E\l[ C^{t}\r | \mathcal{F}^{t-1}]$ as the reward signal.
Accordingly, at Line \ref{line:sampling} of Algorithm \ref{alg:theor1}, a direction $v^t$ is sampled from the current policy, while at Lines \ref{line:g_v} and \ref{line:mu_update} the policy is updated via a gradient ascent step.

Importantly, in practice the derivative of the expected gradient alignment $\nabla_{\mu} \E\l[C^{t}|\mathcal{F}^{t-1}\r]\Big|_{\mu = \mu^{t}}$ is not accessible. Consequently, while the analysis employs the update in Line \ref{line:g_v}, for practical implementation we propose using a well-established policy gradient approach \citep{williams1992simple} from reinforcement learning. We refer to Sections \ref{sec:toy_exp} and \ref{sec:ZO_framework} for further details on this implementation.

\subsection{The dynamics of gradient alignment}

Now we present Theorem \ref{th:main}, which characterizes the evolution of the expected gradient alignment $\E\l[C^{t}|\mathcal{F}^{t-1}\r]$ over a single iteration of Algorithm \ref{alg:theor1} and constitutes the core of the theoretical analysis of the proposed approach.

\begin{theorem}\label{th:main}
    Let Assumption \ref{ass:lip} hold and $d \geq 16$. Then for iterations of Algorithm \ref{alg:theor1} with $\varepsilon \leq \mathcal{O}\l(d^{-\frac{3}{2}}  \cos(\beta^t)(1-\cos(\beta^t)) \|\mu^t\|\r), \gamma_\mu^t \leq \mathcal{O}\l(\|\mu^t\|^2\r),\gamma_x^t \leq \mathcal{O}\l(\frac{\cos^2(\beta^t)(1-\cos(\beta^t))^2}{L}\r)$ the following inequality holds:
    \begin{eqnarray*}
    &&\E\l[ C^{t+1}\r | \mathcal{F}^t]
      =
      \E\l[ C^{t}\r | \mathcal{F}^{t-1}]
      \notag +
     \gamma^t_\mu
     \l(\frac{\cos(\beta^t)(1-\cos(\beta^t))}{24 \|\mu^t\|}\r)^2,
\end{eqnarray*}
where $\beta^t$ is the angle between $\mu^t$ and $\nabla f (x^t)$.
\end{theorem}

Under the assumptions of Theorem \ref{th:main}, the values of $\E\l[ C^{t}\r | \mathcal{F}^{t-1}]$ increase monotonically.
This leads to acceleration in convergence guarantees in Lemma \ref{lem:descent} and corresponds to an improvement of the policy $\mu^t$. 
At the same time, the theorem does not provide guarantees for a naive initialization $\mu^0 = 0$ (Figure \ref{pic:landscape} illustrates that this point corresponds to a saddle point of the function $\E\l[C^{t}|\mathcal{F}^{t-1}\r]$). It is consistent with the observation that $\nabla_{\mu} \E\l[C^{t}|\mathcal{F}^{t-1}\r]\Big|_{\mu = 0} = 0$, together with the update rule in Line $\ref{line:mu_update}$ of Algorithm $\ref{alg:theor1}$, results in unchanged values of the policy: $\mu^{t+1} = \mu^t$ and, consequently, expected gradient alignment $\E\l[C^{t+1}|\mathcal{F}^{t}\r] = \E\l[C^{t}|\mathcal{F}^{t-1}\r]$. We also note that growth of the gradient alignment is not guaranteed in the cases of collinearity or orthogonality between $\mu^t$ and $\nabla f(x^t)$, since the gradient of the alignment objective vanishes in these configurations. However, Algorithm \ref{alg:theor1} almost surely avoids these degenerate configurations when initialized at random, as they form a measure-zero set in the parameter space.

\paragraph{Intuition.} We now provide further intuition for the concepts underlying Theorem \ref{th:main}. Figure \ref{pic:landscape} depicts the landscape of the function $\mathbb{E}[C^{t}\mid \mathcal{F}^{t-1}]$ as a function of $\mu^t$ in dimension $d=2$.

\begin{wrapfigure}[18]{r}{0.45\columnwidth}
    \centering
    \vspace{-1.7em}
    \includegraphics[width=\linewidth]{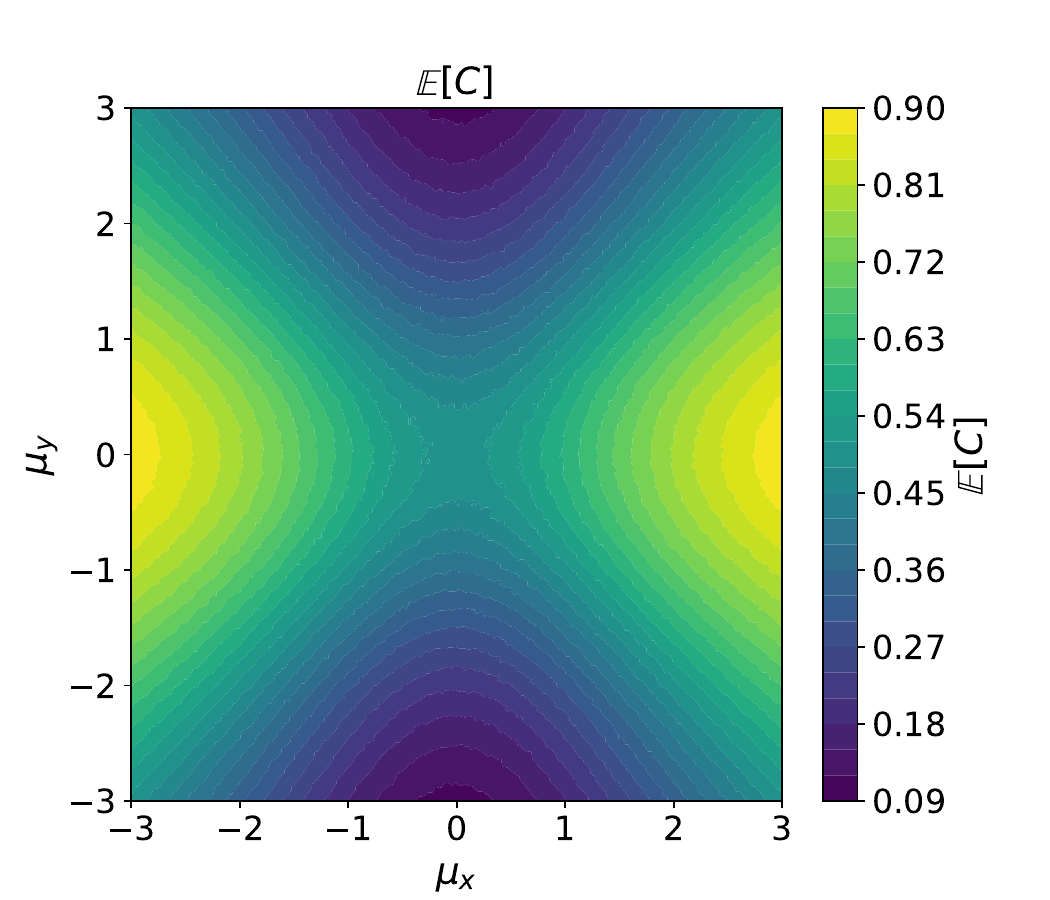}
    \caption{Landscape of the function $\mathbb{E}\left[ C^{t} \mid \mathcal{F}^{t-1}\right]$ with respect to $\mu^t$, for $\nabla f(x^t) = (1,0)^\top$ and $d=2$.}
    \label{pic:landscape}
\end{wrapfigure}
The figure clearly reveals the intuitive saddle-like structure of the objective: increasing orthogonality of $\mu^t$ and $\nabla f(x^t)$ leads to sampling vectors $v^t$ that form large angles with $\nabla f(x^t)$ and therefore yield small values of $C^t$. In contrast, alignment between $\mu^t$ and $\nabla f(x^t)$ results in larger values of both $C^t$ and its conditional expectation.
This landscape additionally exhibits symmetry with respect to the transformation $\mu^t \mapsto -\mu^t$, reflecting the fact that $C^t$ depends only on the squared cosine of the angle $\beta^t$.

Each iteration of Algorithm \ref{alg:theor1} performs a gradient ascent step on the function $\mathbb{E}[C^{t}\mid \mathcal{F}^{t-1}]$ with respect to $\mu^t$. From the RL perspective, it corresponds to a policy update driven by the observed reward signal $\mathbb{E}[C^{t}\mid \mathcal{F}^{t-1}]$.
Subsequently, a model parameter update is performed in Line \ref{line:x_update_th}, which in turn changes the direction of the true gradient $\nabla f(x^t)$. Geometrically, this corresponds to a rotation of the saddle while keeping $\mu^t$ fixed. 
Assumption \ref{ass:lip} on the Lipschitz continuity of $f$, together with the explicit form of the update in Line \ref{line:x_update_th}, allows to bound the angle of this rotation in terms of the step size $\gamma_x^t$. This bounded rotation may also be interpreted as a bounded change of the environment, which may reduce the relevance of the policy obtained at the previous iteration. The analysis in Theorem \ref{th:main} is built upon carefully balancing these two competing effects: the gradient ascent step in $\mu^t$ and the change in the direction of $\nabla f(x^t)$ induced by the update of the model parameters $x^t$.

The saddle structure of $\mathbb{E}[C^{t}\mid \mathcal{F}^{t-1}]$ immediately implies non-concavity. It prevents us from directly deriving its increase in Theorem \ref{th:main} using standard results from gradient ascent theory, thereby significantly complicating the theoretical analysis. Due to the non-standard challenges and the resulting proof technique, we provide both a high-level sketch and a complete proof in Appendix \ref{sec:mis_proofs}.


To characterize the full dynamics of $\E\l[ C^{t}\r | \mathcal{F}^{t-1}]$ across the iterations of Algorithm \ref{alg:theor1} we prove the following result.
\begin{lemma}\label{lem:multiitteration_eval}
    Let in the setting of Theorem \ref{th:main} $\delta \leq \cos(\beta^0)\leq 1-\delta$  with $\delta \in \l(0;\frac{1}{2}\r)$, $\|\mu^t\|\leq M$ for all $t$ and $\varepsilon = \mathcal{O}\l(d^{-\frac{3}{2}}  \delta M\r),\gamma_x^t = \mathcal{O}\l(\frac{\delta^2}{L}\r)$. Then $\E\l[ C^t|\mathcal{F}^{t-1}\r]$ increase monotonically to the value $$\cos\l(\frac{\delta}{32d} + L\gamma_x^t + \arccos(1-\delta)\r)^2 (1-e^{-1}) = \mathcal{O}(1)$$ and then never decrease below.
\end{lemma}
Lemma \ref{lem:multiitteration_eval} establishes once the policy is initialized such that $\delta \leq \cos(\beta^0)\leq 1-\delta$, for an appropriate choice of $\gamma_x^t, \gamma_\mu^t,\varepsilon$ the expected gradient alignment $\E\l[ C^{t}\r | \mathcal{F}^{t-1}]$ increases monotonically up to a certain level and subsequently oscillates within a neighbourhood of $1$. This stands in contrast to the $\mathcal{O}(1/d)$ behaviour established in Corollary \ref{lem:ZO_SGD}.


\subsection{Convergence Guaranties}

Lemma \ref{lem:multiitteration_eval} naturally leads to two complementary settings for policy initialization $\mu^0$. The first corresponds to choosing $\mu^0$ randomly without any prior knowledge of the target function $f$.
This setup is analysed in Appendix \ref{sec:no_prior_init} and can be interpreted as a sequential growth of gradient alignment up to a saturation level. In particular, it leads to a non-standard convergence criterion formulated in terms of a weighted average of squared gradient norms evaluated along the iterates of the algorithm.

The second regime for $\mu^0$ initialization corresponds to leveraging prior information about the target function $f$ (the direction of the true gradient $\nabla f(x^0)$ at the initial point) to immediately operate in the phase of Lemma \ref{lem:multiitteration_eval} where $\E\l[ C^{t}\r | \mathcal{F}^{t-1}]$ oscillates in a neighbourhood of $1$. 
\begin{lemma}\label{lem:parallel}
    Let Assumptions \ref{ass:lip}, \ref{ass:minimizer} be satisfied, $d \geq 16$  and $\mu^0$ collinear to $\nabla f(x^0)$. Then for iterations of Algorithm \ref{alg:theor1} with $\delta \leq \cos(\beta^0)\leq 1-\delta$, $\delta \in \l(0;\frac{1}{2}\r)$ and $\varepsilon = \mathcal{O}\l(d^{-\frac{3}{2}}  \delta M\r),M = \min_{t}\{\|\mu^t\|\},\gamma_x^t = \mathcal{O}\l(\frac{\delta^2}{L}\r),\gamma_\mu^t = \mathcal{O}\l(\|\mu^t\|^2\r)$, the following bound holds:
    $$\E[\|\nabla f(\overline{x}^T)\|^2] \leq \mathcal{O}\l( \frac{1}{T\gamma_x}\r) (f(x^0) - f(x^*)).$$
\end{lemma}
    
\paragraph{Discussion.} We note that the dimension $d$ does not appear explicitly in the theoretical bounds of Lemma \ref{lem:parallel}. In contrast, classical directional, coordinate or ZO methods incur the polynomial dependence on $d$ in their convergence guarantees \citep{nesterov2017random, veprikov2024new,hanzely2018sega}. This slowdown is in fact unavoidable for uniformly random directions from the perspective of lower bounds \citep{lee2013efficient}, which establish an inherent $\mathcal{O}(d)$ factor for standard ZO methods. Our approach overcomes this barrier by learning the parameters of the sampling distribution, effectively steering the sampled directions to better align with the gradient and thereby removing the explicit dependence on $d$ in the convergence rate.
    
It is worth noting that, in theory, the quality of the Monte-Carlo approximation $g_\mu^t$ obtained with a fixed number of samples $v^t$ per iteration depends on the dimension $d$. However, in practice a small number of samples, often $\mathcal{O}(1)$, is sufficient to obtain a stable Monte-Carlo estimate \citep{schulman2017proximal}.

Finally, we observe that after normalization the sampling distribution is invariant under simultaneous rescaling of the mean vector $\mu$ and the variance $\varepsilon^2$.
This invariance also appears in choice of $\varepsilon \sim d^{-\frac{3}{2}}  \delta M \leq d^{-\frac{3}{2}}  \delta \|\mu^t\|$ in the analysis. This suggests that constraining $\|\mu\| = 1$ may be a natural design choice and a direction for future work.

\subsection{Toy Experiment}\label{sec:toy_exp}
To validate the theoretical results, we conduct a toy experiment consisting of training a linear regression model on the a9a \citep{chang2011libsvm} dataset with access to directional derivatives. As a baseline, we consider an algorithm that employs the estimator with $K$ Monte Carlo samples:
\begin{eqnarray}\label{eq:toy_estimator}
    g_x^t = \frac{1}{K} \sum_{k=1}^K \overline v_k^t \langle \overline v_k^t, \nabla f(x^t)\rangle,
\end{eqnarray}
with $v_k^t \sim \mathcal{N}(0,I)$. 
For \texttt{LDSD}, we use the same approximation \eqref{eq:toy_estimator} with the choice $v_k^t \sim \mathcal{N}(\mu^t,\varepsilon^2 I)$. For proposed learning mechanism we apply a log-derivative trick \citep{williams1992simple} to derive a Monte Carlo estimator of $\nabla_{\mu} \E\l[C^{t}|\mathcal{F}^{t-1}\r]\Big|_{\mu = \mu^{t}}$ from the Line \ref{line:g_v} of Algorithm \ref{alg:theor1}. 
We denote the density of a Gaussian distribution by $\pi_\mu(\cdot)$ and introduce $C_k^t = \langle{\overline{v_k^t}, \overline{\nabla f(x^t)}}\rangle^2$ for each sampled direction $v^t_k$. Accordingly,
\begin{align*}
    \nabla_{\mu}\E\l[C^{t}|\mathcal{F}^{t-1}\r]\Big|_{\mu = \mu^{t}} &= \int C^t \l(\nabla_\mu \log\pi_{\mu^t}\r) \mathrm{d}\pi_{\mu^t} \notag\approx \frac{1}{K}\sum_{k=1}^K C_k^t \l(\nabla_\mu \log\pi_{\mu^t}(v_k^t) \r).
\end{align*}

Since $\pi_{\mu^t}(\cdot)$ is a Gaussian distribution density, we have $\nabla_\mu \log\pi_{\mu^t}(v_k^t) = \frac{v_k^t - \mu^t}{\varepsilon^2}$. For policy updates we utilize 
\begin{eqnarray*}
    g_\mu^t = \frac{1}{K}\sum_{k=1}^K\l( C^t_k - b^t\r) \frac{v_k^t - \mu^t}{\varepsilon^2},
\end{eqnarray*}
where $b^t = \frac{1}{K}\sum_{k=1}^K C^t_k$ is a mean-baseline, introduced to reduce the estimator variance \cite{williams1992simple}.

We set $K=5$ both for the baseline and \texttt{LDSD} and compare them in terms of the alignment between the true gradient and its estimator $\cos(g_x^t, \nabla f(x^t))$, as well as the norm of the gradient $\| \nabla f(x^t ) \|$: the quantity involved in convergence guaranties. The results are reported in Figure \ref{fig:toy}. 

\begin{figure}[h]
    \centering
    \begin{subfigure}
        \centering
        \includegraphics[width=0.48\linewidth]{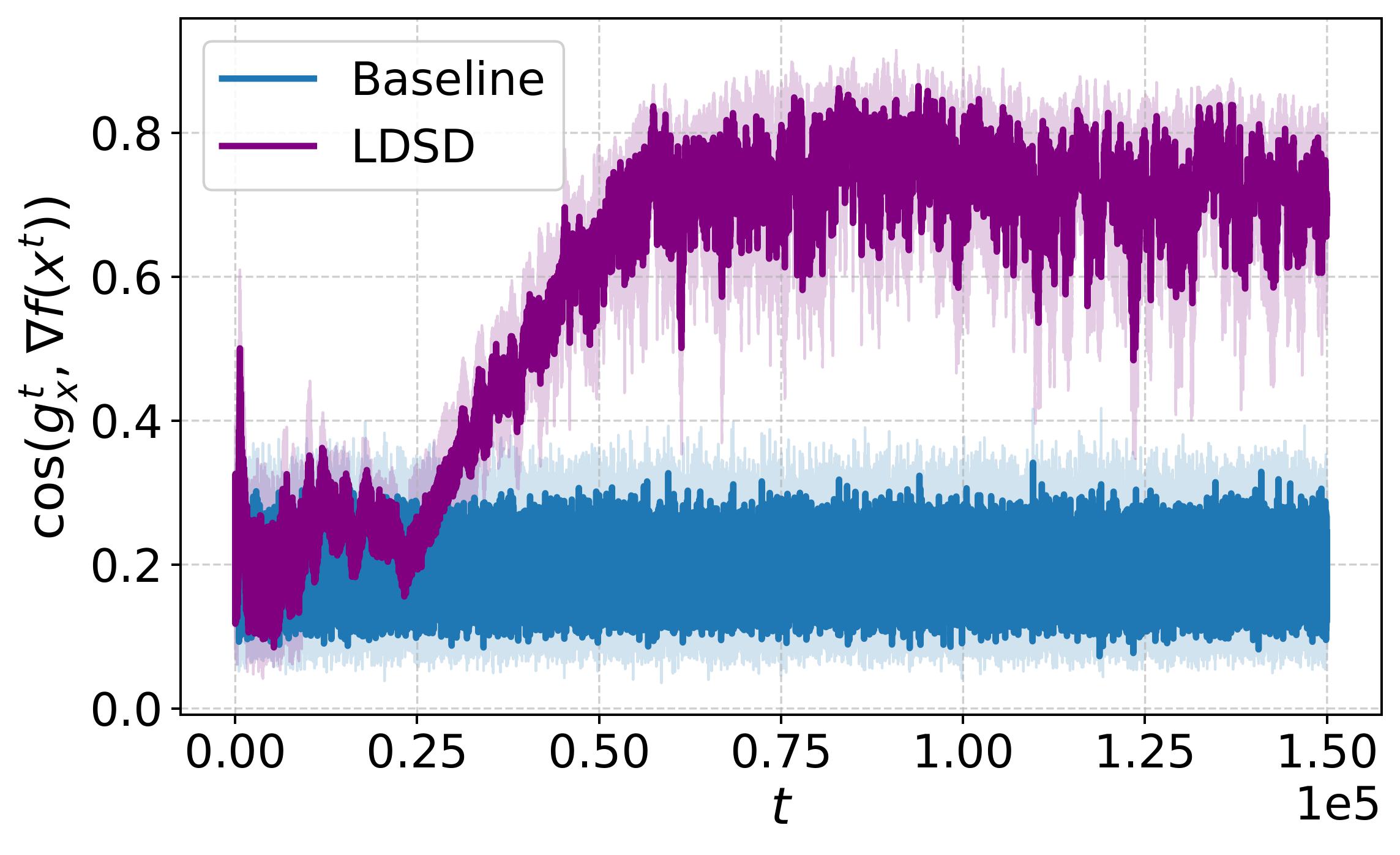}
    \end{subfigure}
    \hfill
    \begin{subfigure}
        \centering
        \includegraphics[width=0.48\linewidth]{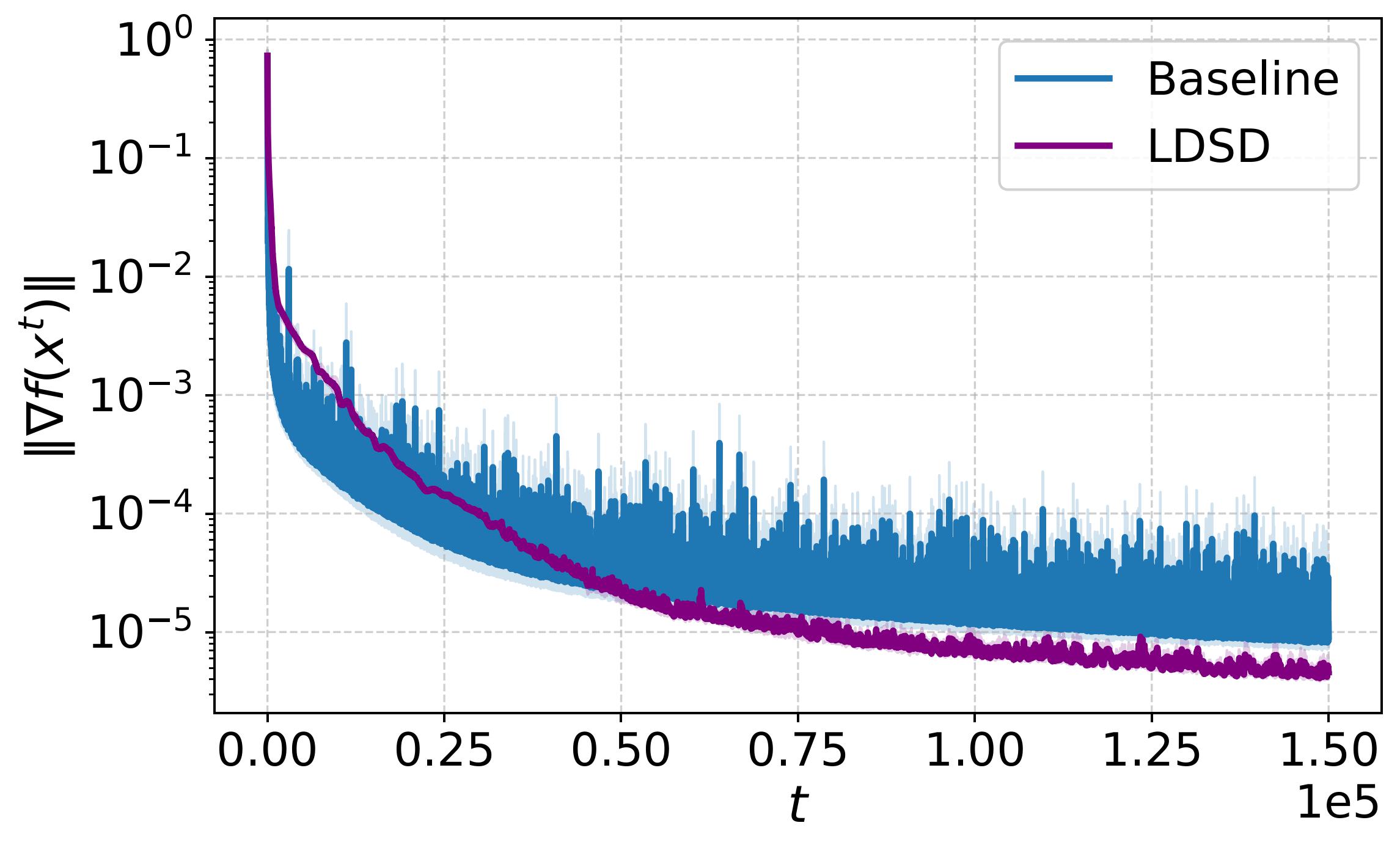}
    \end{subfigure}

    \caption{Comparison of \texttt{LDSD} and the baseline on the a9a regression task.}
    \label{fig:toy}
\end{figure}

The first plot in Figure \ref{fig:toy} reveals learning the sampling direction in \texttt{LDSD} indeed leads to strong alignment between the estimator $g_\mu^t$ and the true gradient $\nabla f(x^t)$. This behaviour fully agrees with Lemma \ref{lem:multiitteration_eval} and the intuition discussed in Section \ref{sec:algo}: the alignment initially increases monotonically and subsequently oscillates in a neighbourhood of $\mathcal{O}(1)$. 

\section{Zero-order Framework}\label{sec:ZO_framework}
Building on Algorithm \ref{alg:theor1}, which presented the first-order instantiation of our update rules, we now develop the corresponding ZO formulation in Algorithm~\ref{alg:zo_framework} by replacing the directional derivative $g_x^t$ with a zero-order oracle (\ref{eq:zo_oracle}), and the gradient $g_\mu^t$ with its REINForCE estimation \cite{williams1992simple}. Both substitutions yield unbiased estimators of the corresponding first-order gradients \cite{feng2023stochastic}, thereby preserving the theoretical structure of the original algorithm in a purely zero-order setting.

\begin{algorithm}{\textsc{ZO-LDSD}}\label{alg:zo_framework}
\begin{algorithmic}[1]
\State \textbf{Input:} Initial points $x^0, \mu^0$, number of iterations $T$, number of oracle calls $K$, variance $\varepsilon$, stepsizes $\gamma_x, \gamma_\mu$
\For{$t = 0,1,2,\dots, T-1$}
    \State Sample $\{v^i\}_{i=1}^K \sim \mathcal{N}(\mu^t, \varepsilon^2I)$ 
    \State Set $v_*^t = \arg\min\limits_{v\in \{v^i\}_{i=1}^K} \l\{f(x^t + \tau v)\r\}$ \label{line:v_star}
    \State $g_x^t = \frac{f(x^t + \tau v_*^t) - f(x^t - \tau v_*^t)}{2\tau}\cdot v_*^t$
    \State $g_{\mu}^t = \frac{1}{K}\sum\limits_{i=1}^K \frac{K\cdot f(x^t + \tau v^i) - \sum\limits_{j=1}^{K} f(x^t + \tau v^j)}{K-1} \cdot \frac{v^i - \mu^t}{\varepsilon^2}$ \label{line:g_mu}
    \State $x^{t+1} = x^t - \gamma_x \cdot g_{x}^t$ \label{line:x_update}
    \State $\mu^{t+1} = \mu^t + \gamma_\mu \cdot g_{\mu}^t$
\EndFor
\end{algorithmic}
\end{algorithm}

The selection of $v_*^t$ in Line \ref{line:v_star} of Algorithm \ref{alg:zo_framework} implements a lightweight, direction-wise greedy search among the $K$ sampled candidates: by choosing the perturbation that empirically yields the smallest objective after a forward evaluation, the method biases the subsequent $x$-update (Line \ref{line:x_update} of Algorithm \ref{alg:zo_framework}) toward directions that have demonstrated the strongest instantaneous decrease in $f$. This selection improves the practical utility of a limited number of oracle queries and typically reduces the variance of the chosen search direction, at the cost of introducing an intentional selection bias in the sampled direction for the $x$-update. 

Aside from those mechanics and the replacement of true gradients by zero-order estimates, Algorithm~\ref{alg:zo_framework} preserves the structural update logic of Algorithm~\ref{alg:theor1}. The procedure preserves the forward-only nature of ZO methods, requires $K+1$ forward evaluations per iteration, and concentrates sampling mass on empirically useful directions through online updates of $\mu^t$. While classical ZO schemes typically use $K=1$, we set $K=5$ in our experiments as an empirically chosen compromise between selection quality and oracle cost. 
An ablation over $K$ is provided in Section~\ref{sec:abl}.

It is important to note that this framework is not a standalone optimizer but a plug-in module: any existing ZO estimator for the parameter update can be combined with the adaptive sampling routine to improve empirical convergence. In particular, the update at line \ref{line:x_update} of Algorithm \ref{alg:zo_framework} is modular and may be replaced by the parameter-update rule of any other ZO optimizer. This enables straightforward integration of the adaptive sampling mechanism, while leaving the base optimizer's update logic intact.

\section{Experiments} \label{sec:llm_exps}

In this section, we present an empirical evaluation of the proposed zero-order optimization framework for large language model fine-tuning.

\subsection{Experimental setup}

\textbf{Fine-Tuning Task and Schemes.} We evaluate the proposed zero-order methods on the SST-2 sentiment classification benchmark \citep{socher2013recursive}, following the experimental protocol of \citep{zo_bench}. SST-2 is a standard supervised benchmark for sentence-level binary sentiment classification and is commonly used to assess fine-tuning procedures for language models.

\textbf{Models.} Experiments are performed on two representative models: RoBERTa-Large \citep{liu2019roberta} and OPT-1.3B \citep{zhang2022opt}. These models span different architectures and parameter counts, allowing us to examine the behavior of ZO algorithms across realistic fine-tuning settings while keeping the experimental scope focused.

\textbf{Baselines.} We use our framework on established zero-order baselines reported in \citep{zo_bench}: ZO-SGD \citep{ghadimi2013stochastic}, ZO-AdaMM \citep{chen2019zo}, JAGUAR SignSGD \cite{petrov2025leveraging}. All experiments follow the evaluation of \citep{zo_bench} to ensure a fair comparison. Training details are deferred to Appendix \ref{subsec:llm_setup_appendix}.

\textbf{Comparison procedure.} For a fair and controlled evaluation we compare each baseline in two variants: (i) the original baseline implementation, and (ii) the same baseline augmented with our zero-order adaptive sampling framework (Algorithm \ref{alg:zo_framework}). Crucially, in (ii) we do not change any of the original optimizer hyperparameters from (i) (e.g., learning rate, momentum), those remain exactly as in the base method. The only hyperparameters we tune are those introduced by our framework. 

Since Algorithm \ref{alg:zo_framework} performs $6$ ZO oracle calls (forward passes) when $K=5$, we compare baseline methods under a fixed budget of oracle calls rather than a fixed number of iterations to ensure fair evaluation. This leads to two natural comparison strategies for the baselines: either perform $6$ oracle calls per iteration (using $K=5$ samples), or perform more iterations with $2$ oracle calls per iteration (using $K=1$). Both variants are presented in Table~\ref{tab:results}. This design ensures that performance differences arise from the adaptive sampling module itself rather than from differences in computational budget. 
\subsection{Results}
Table \ref{tab:results} reports SST2 test accuracy for OPT-1.3B and RoBERTa-Large under different fine-tuning schemes. 

\begin{table}[h]
    \centering
    \caption{Test accuracy on SST2 for OPT-1.3B and RoBERTa-Large. Best performance among ZO methods is in \textbf{bold}.}
    \resizebox{0.7\linewidth}{!}{
    \begin{tabular}{|l|l|cc|cc|}
    \toprule
    Method & Sampling & \multicolumn{2}{c|}{OPT-1.3B} & \multicolumn{2}{c|}{RoBERTa-Large} \\
    \cmidrule(lr){3-4} \cmidrule(lr){5-6}
    & & FT & LoRA & FT & LoRA \\
    \midrule
    \multirow{5}{*}{ZO-SGD} 
    & Gaussian, & \multirow{2}{*}{0.928} & \multirow{2}{*}{0.854} & \multirow{2}{*}{0.899} & \multirow{2}{*}{0.913} \\
    & 2 forwards, more iterations & & & & \\
    \cmidrule(lr){2-6}
    & Gaussian, & \multirow{2}{*}{0.919} & \multirow{2}{*}{0.850} & \multirow{2}{*}{0.869} & \multirow{2}{*}{0.905} \\
    & 6 forwards, same iterations & & & & \\
    \cmidrule(lr){2-6}
    & \cellcolor{bgcolor2}\textbf{Algorithm~\ref{alg:zo_framework}} & \cellcolor{bgcolor2}\textbf{0.935} & \cellcolor{bgcolor2}\textbf{0.916} & \cellcolor{bgcolor2}\textbf{0.909} & \cellcolor{bgcolor2}\textbf{0.928}
    \\
    \midrule
    \multirow{5}{*}{ZO-AdaMM}
    & Gaussian, & \multirow{2}{*}{0.925} & \multirow{2}{*}{0.912} & \multirow{2}{*}{0.909} & \multirow{2}{*}{0.812} \\
    & 2 forwards, more iterations & & & & \\
    \cmidrule(lr){2-6}
    & Gaussian, & \multirow{2}{*}{0.917} & \multirow{2}{*}{0.905} & \multirow{2}{*}{0.899} & \multirow{2}{*}{0.801} \\
    & 6 forwards, same iterations & & & & \\
    \cmidrule(lr){2-6}
    & \cellcolor{bgcolor2}\textbf{Algorithm~\ref{alg:zo_framework}} & \cellcolor{bgcolor2}\textbf{0.930} & \cellcolor{bgcolor2}\textbf{0.929} & \cellcolor{bgcolor2}\textbf{0.917} & \cellcolor{bgcolor2}\textbf{0.900} \\
    \midrule
    \multirow{5}{*}{JAGUAR}
    & Gaussian, & \multirow{2}{*}{0.925} & \multirow{2}{*}{0.921} & \multirow{2}{*}{0.906} & \multirow{2}{*}{0.905} \\
    & 2 forwards, more iterations & & & & \\
    \cmidrule(lr){2-6}
    & Gaussian, & \multirow{2}{*}{0.914} & \multirow{2}{*}{0.911} & \multirow{2}{*}{0.896} & \multirow{2}{*}{0.893} \\
    & 6 forwards, same iterations & & & & \\
    \cmidrule(lr){2-6}
    & \cellcolor{bgcolor2}\textbf{Algorithm~\ref{alg:zo_framework}} & \cellcolor{bgcolor2}\textbf{0.930} & \cellcolor{bgcolor2}\textbf{0.928} & \cellcolor{bgcolor2}\textbf{0.922} & \cellcolor{bgcolor2}\textbf{0.927} \\
    \bottomrule
    \end{tabular}
    }
    \label{tab:results}
\end{table}

\textbf{Discussion.} 
The results show that replacing Gaussian sampling with our adaptive direction sampling produces consistent improvements across models, optimizers, and fine-tuning modalities. Notably, merely increasing the number of probes with Gaussian sampling does not yield the same benefit: performance gains arise from the combination of multiple candidates and adaptive selection. These findings indicate that our sampling scheme yields more informative finite-difference probes and reduces variance in the resulting ZO gradient proxies. Because the framework is plug-and-play, it improves practical convergence of existing ZO methods without altering their core update rules.



\subsection{Ablation Studies}\label{sec:abl}

In this section, we provide several ablation studies on hyperparameters from Algorithm \ref{alg:zo_framework}. In all experiments we use SST-2 dataset with the RoBERTa-large model. Figure~\ref{fig:zo_hparam_ablation}(a) shows accuracy rising with $K$ up to $K=5$ and then slightly falling; (b) exhibits improvement as $\gamma_{\mu}$ increases until an intermediate optimum after which performance degrades;  and (c) displays a U-shaped dependence on $\varepsilon$ with a clear peak where our method outperforms the Gaussian baseline. These results motivate the default hyperparameter choices used in our experiments.



\makeatletter
\iftwocolumnmode
  \def\brainFigWidth{0.32\columnwidth}
\else
  \def\brainFigWidth{0.32\linewidth}
\fi
\makeatother

\begin{figure}[H]
  \centering
  \begin{minipage}[t]{\brainFigWidth}
    \centering
    \includegraphics[width=\linewidth,height=4.2cm,keepaspectratio]{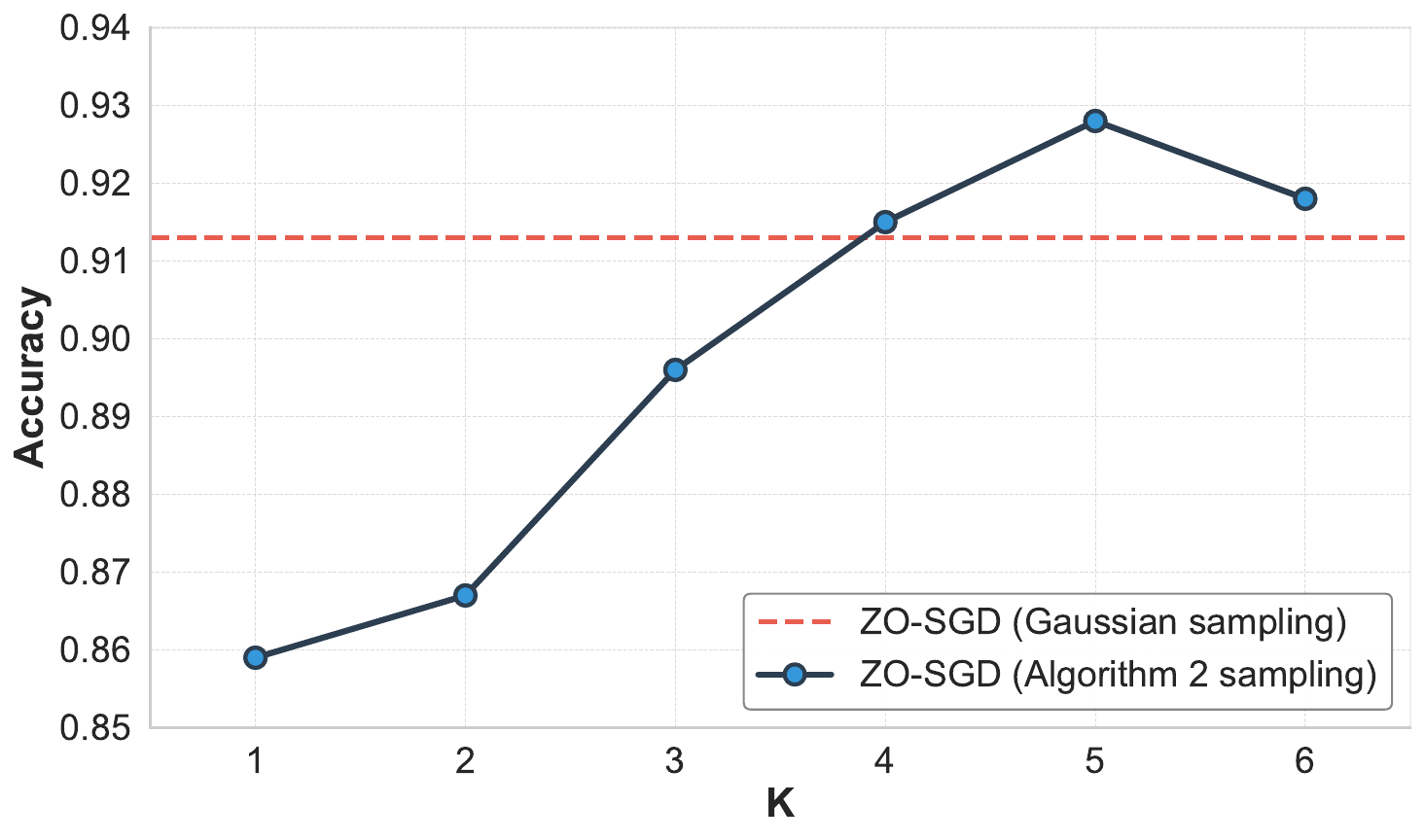}%
    \\[0.5ex]{\small\textbf{(a)} $K$}\label{fig:k_ablation}
  \end{minipage}\hfill
  \begin{minipage}[t]{\brainFigWidth}
    \centering
    \includegraphics[width=\linewidth,height=4.2cm,keepaspectratio]{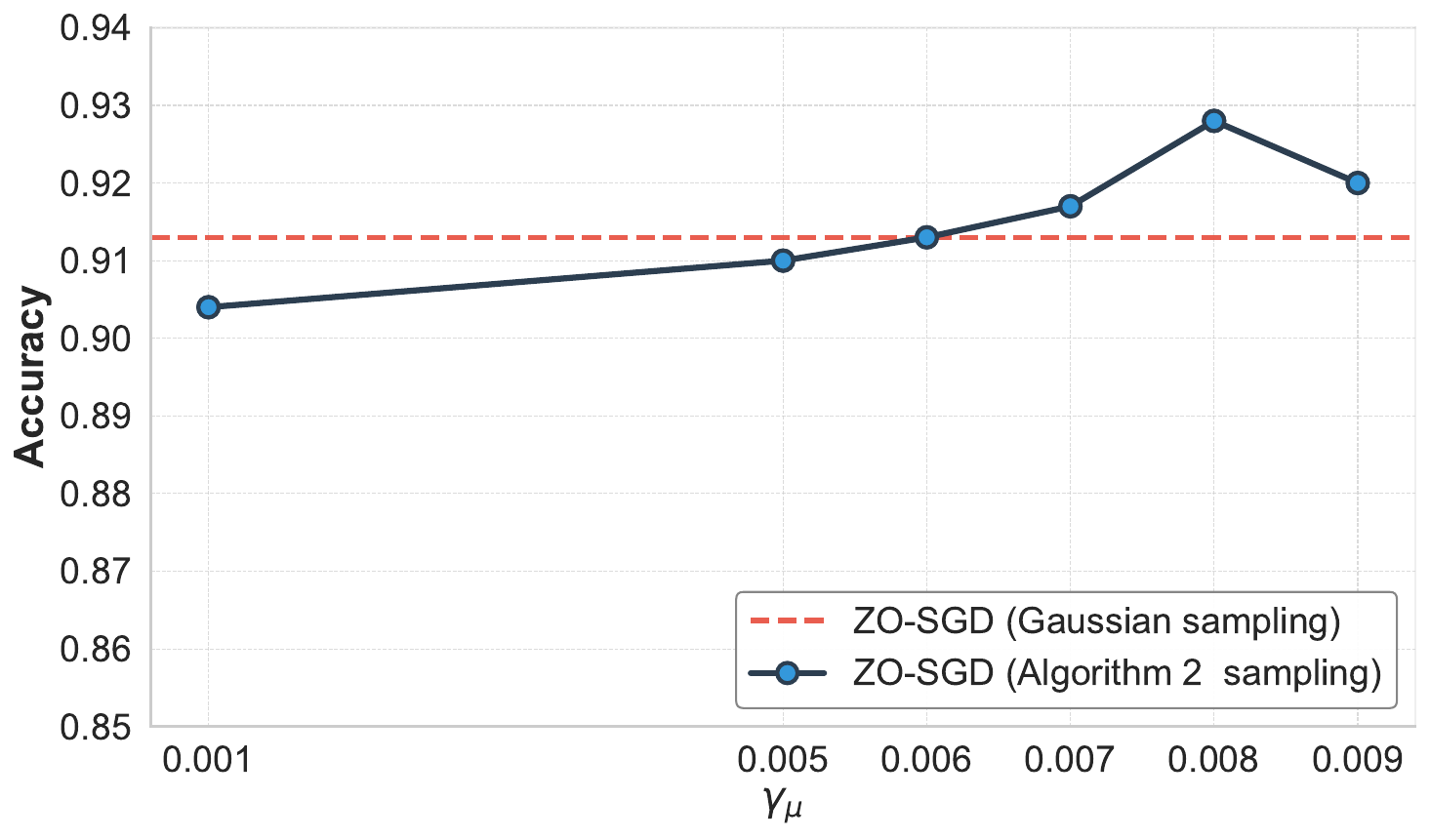}%
    \\[0.5ex]{\small\textbf{(b)} $\gamma_\mu$}\label{fig:lrmu_ablation}
  \end{minipage}\hfill
  \begin{minipage}[t]{\brainFigWidth}
    \centering
    \includegraphics[width=\linewidth,height=4.2cm,keepaspectratio]{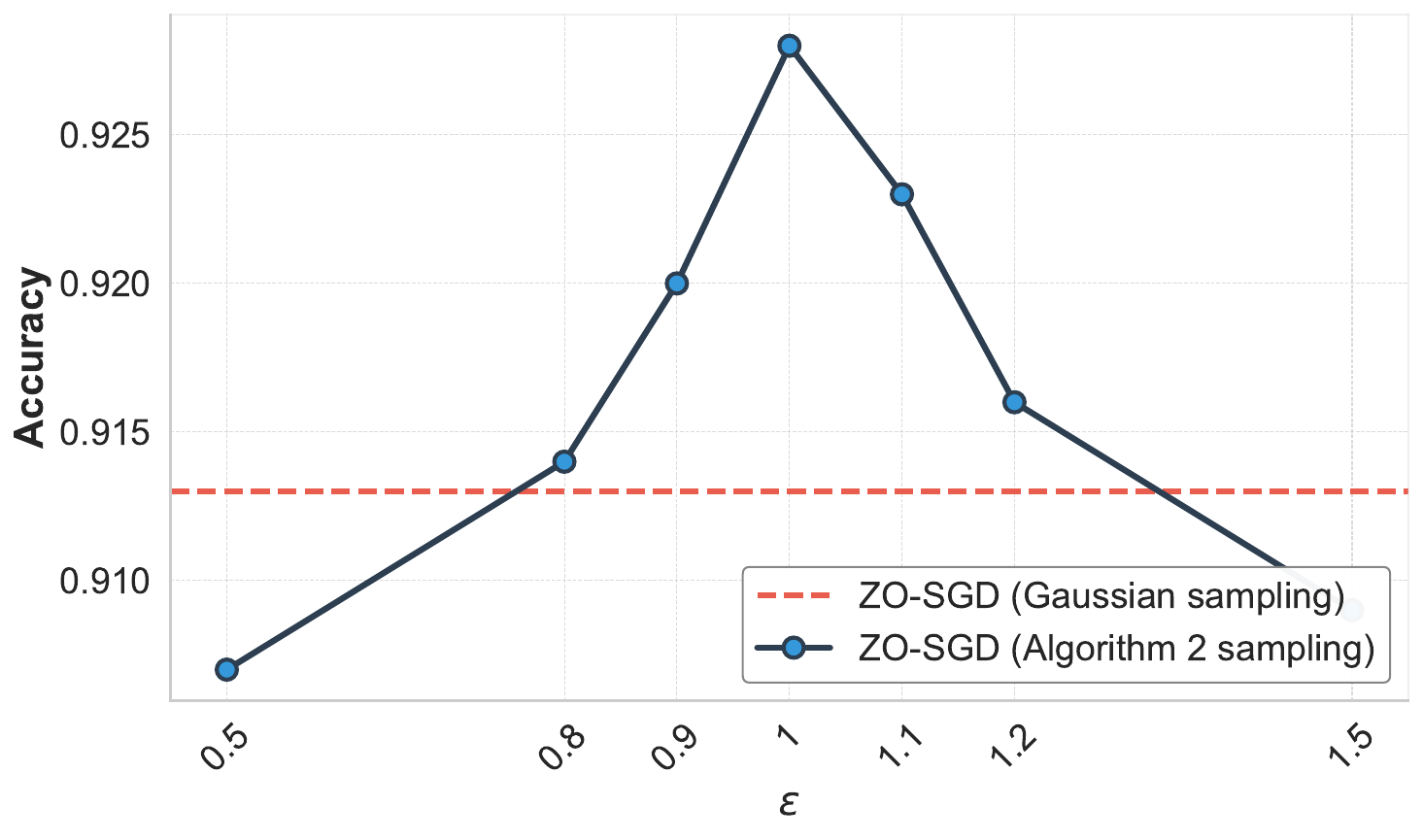}%
    \\[0.5ex]{\small\textbf{(c)} $\varepsilon$}\label{fig:eps_ablation}
  \end{minipage}

  \caption{Test accuracy of ZO-SGD (Algorithm \ref{alg:zo_framework} sampling) on SST-2 RoBERTa-large with LoRA for different hyperparameters.}
  \label{fig:zo_hparam_ablation}
\end{figure}




\section{Conclusion}
In this work, we introduce a novel zero-order optimization framework that treats the sampling distribution over perturbation directions as a learnable policy. Theoretically, we provide convergence guarantees with no explicit dependence on dimension $d$, in stark contrast to the $\mathcal{O}(d)$ degradation inherent in classical zero-order methods. This represents a significant advancement for ZO optimization, achieved by adaptively learning the mean of the sampling distribution to maintain high gradient alignment throughout training.

Empirically, we validate our framework on LLM fine-tuning benchmarks, demonstrating consistent improvements over established ZO baselines including ZO-SGD, ZO-AdaMM, and JAGUAR SignSGD. Our framework operates as a plug-and-play module that enhances existing ZO optimizers without modifying their core hyperparameters. Extensive ablation studies confirm the robustness and practical applicability of our approach.


\end{mainpart}

\begin{appendixpart}

\section{Additional Experimental Details}\label{sec:additional_exp_details}

\subsection{Toy experiment}
The tuned hyperparameters used in toy experiment are as follows:

\begin{itemize}
    \item \textbf{Baseline DGD:} 
    \begin{itemize}
        \item $\gamma_x^t = 200$.
    \end{itemize}
    \item \textbf{LDSD:} 
    \begin{itemize}
        \item $\gamma_x^t = 5$,
        \item $\gamma_\mu^t = 1.4 \cdot 10^{-5}$,
        \item $\varepsilon = 1.2 \cdot 10^{-2}$.
    \end{itemize}
\end{itemize}

\subsection{Experiments on LLM fine-tuning}
\label{subsec:llm_setup_appendix}

All experiments in Section~\ref{sec:llm_exps} used the same core sampling hyperparameters: \(\varepsilon=1\) and \(\gamma_{\mu}=10^{-3}\). For ZO-SGD setups we fixed momentum \(=0.9\). For ZO-Adamm setups we fixed $(\beta_1, \beta_2) = (0.9, 0.999)$. For ZO-Adamm setups we fixed JAGUAR momentum $\beta=0.9$. All methods employed a cosine learning-rate schedule for \(\gamma_x\). Table~\ref{tab:hp_table} reports the base learning rates \(\gamma_x\) used for each method and experimental condition. 

\begin{table}[h]
    \centering
    \caption{Base learning rates \(\gamma_x\) used for OPT-1.3B and RoBERTa-Large experiments with FT and LoRA.}
    \resizebox{\linewidth}{!}{
    \begin{tabular}{l|cc|cc}
    \toprule
    Method & \multicolumn{2}{c}{OPT-1.3B} & \multicolumn{2}{c}{RoBERTa-Large} \\
    \cmidrule(lr){2-3} \cmidrule(lr){4-5}
    & FT & LoRA & FT & LoRA \\
    \midrule
 
    ZO-SGD (Gaussian sampling, 2 forwards, more iterations) & \(4.0\times10^{-8}\) & \(1.0\times10^{-6}\) & \(4.0\times10^{-8}\) & \(5.0\times10^{-4}\) \\
    ZO-SGD (Gaussian sampling, 6 forwards, same iterations) & \(4.0\times10^{-8}\) & \(1.0\times10^{-6}\) & \(4.0\times10^{-8}\) & \(5.0\times10^{-4}\) \\
    ZO-SGD (Algorithm \ref{alg:zo_framework} sampling) & \(4.0\times10^{-8}\) & \(1.0\times10^{-6}\) & \(4.0\times10^{-8}\) & \(5.0\times10^{-4}\) \\
    \midrule
    ZO-AdaMM (Gaussian sampling, 2 forwards, more iterations) & \(1.0\times10^{-5}\) & \(6.0\times10^{-5}\) & \(3.0\times10^{-5}\) & \(3.0\times10^{-3}\) \\
    ZO-AdaMM (Gaussian sampling, 6 forwards, same iterations) & \(1.0\times10^{-5}\) & \(6.0\times10^{-5}\) & \(3.0\times10^{-5}\) & \(3.0\times10^{-3}\) \\
    ZO-AdaMM (Algorithm \ref{alg:zo_framework} sampling) & \(1.0\times10^{-5}\) & \(6.0\times10^{-5}\) & \(3.0\times10^{-5}\) & \(3.0\times10^{-3}\) \\
     \midrule
    JAGUAR SignSGD (Gaussian sampling, 2 forwards, more iterations) & \(2.0\times10^{-6}\) & \(1.0\times10^{-5}\) & \(4.0\times10^{-6}\) & \(3.0\times10^{-4}\) \\
    JAGUAR SignSGD (Gaussian sampling, 6 forwards, same iterations) & \(2.0\times10^{-6}\) & \(1.0\times10^{-5}\) & \(4.0\times10^{-6}\) & \(3.0\times10^{-4}\) \\
    JAGUAR SignSGD (Algorithm \ref{alg:zo_framework} sampling) & \(2.0\times10^{-6}\) & \(1.0\times10^{-5}\) & \(4.0\times10^{-6}\) & \(3.0\times10^{-4}\) \\
    \bottomrule
    \end{tabular}
    }
    \label{tab:hp_table}
\end{table}

\section{Missing proofs}\label{sec:mis_proofs}

\begin{lemmarepeat}{lem:descent}
Under Assumptions \ref{ass:lip} and \ref{ass:minimizer} for iterations of the form 
$$x^{t+1} = x^t - \gamma_x^t \overline{v^t}\langle \overline{v^t}, \nabla f(x^t)\rangle,$$
the following inequality holds.
    \begin{eqnarray*}
    &&\sum_{t=0}^{T-1} \frac{1}{T} \E\l[ \l(\gamma_x^t - \frac{L {\gamma_x^t}^2}{2} \r)\E\l[C^{t}|\mathcal{F}^{t-1}\r]\|\nabla f(x^{t})\|^2 \r] \notag \\
    &&~~~~~~~~~~~~~~~~~~~~~~~~~~~~\leq \frac{f(x^0)-f(x^{*})}{T}.
\end{eqnarray*}
\end{lemmarepeat}
\begin{proof}
    We consider the algorithm step and utilize Assumption \ref{ass:lip}.
    \begin{eqnarray*}
        f(x^{t+1})-f(x^{t}) 
        &=&
        f(x^t-\gamma_x^t\overline v^t \langle\nabla f(x^t),\overline v^t\rangle) - f(x^t)\notag\\
        &\overset{\ref{ass:lip}}{\leq}& 
        -\gamma_x^t\langle\nabla f(x^t),\overline v^t\rangle^2+\frac{L}{2}\|{\gamma_x^t}\overline v^t\langle\nabla f(x^t),\overline v^t\rangle\|^2 \notag \\
        &=&\l( \frac{L {\gamma_x^t}^2}{2} - \gamma_x^t\r)\langle\nabla f(x^t),\overline v^t\rangle^2 \notag \\
        &=&\l( \frac{L {\gamma_x^t}^2}{2} - \gamma_x^t\r)\langle \overline{\nabla f(x^t)},\overline v^t\rangle^2 \|\nabla f(x^t)\|^2 \notag
        .
    \end{eqnarray*}

Conditioning on $\mathcal{F}^t$ and using the $\mathcal{F}^{t-1}$-measurability of $\nabla f(x^t)$ and $\gamma_x^t$, we obtain
\begin{eqnarray*}
    \E\l[f(x^{t+1})|\mathcal{F}^{t-1}\r]-f(x^{t}) 
    \leq \l( \frac{L {\gamma_x^t}^2}{2} - \gamma_x^t\r)\E\l[C^{t}|\mathcal{F}^{t-1}\r]\|\nabla f(x^{t})\|^2.
\end{eqnarray*}
Summing over $t = 0,\dots,T-1$ and considering full expectation we derive
\begin{eqnarray*}
    \sum_{t=0}^{T-1} \frac{1}{T} \E\l[ \l(\gamma_x^t- \frac{L {\gamma_x^t}^2}{2}\r)\E\l[C^{t}|\mathcal{F}^{t-1}\r]\|\nabla f(x^{t})\|^2 \r] \leq \frac{f(x^0)-\E[f(x^{t-1})]}{T}.
\end{eqnarray*}
Applying Assumption \ref{ass:minimizer} finishes the proof.
\end{proof}

\begin{corollaryrepeat}{lem:ZO_SGD}
    Let in the setting of Lemma \ref{lem:descent} $v^t \overset{i.i.d.}{\sim} RS(1)^d_{\|\cdot\|}$ be distributed uniformly on the unit sphere and $\gamma^t_x \leq \frac{1}{2L}$. Then the following bound holds:
    \begin{eqnarray*}
        \E \| \nabla f(\overline{x}^T)\|^2 \leq 
        \frac{2d}{\gamma_x^t T}(f(x^0) - f(x^*)),
    \end{eqnarray*}
    where $\overline{x}^T$ is chosen uniformly from $\{x^0,x^1,\dots,x^{t-1}\}$.
\end{corollaryrepeat}
\begin{proof}
    The choice $\gamma^t_x \leq \frac{1}{2L}$ guarantees:
    \begin{eqnarray}\label{eq:zosgd_1}
        \gamma_x^t- \frac{L {\gamma_x^t}^2}{2} \geq \frac{\gamma_x^t}{2}
    \end{eqnarray}
    Then we need to evaluate $\E\l[C^{t}|\mathcal{F}^{t-1}\r]$. We consider coordinate system where $e_1 ~\|~ \nabla f(x^t)$:
    \begin{eqnarray*}
        \E\l[C^{t}|\mathcal{F}^{t-1}\r] = \E_{v^{t} \sim \mathcal{N}(0,I)}\l[ \frac{v_1^2}{\|v\|^2}\r].
    \end{eqnarray*}
    Due to the distribution symmetry we conclude:
    \begin{eqnarray*}
        \E_{v^{t} \sim \mathcal{N}(0,I)}\l[ \frac{v_1^2}{\|v\|^2}\r] = \E_{v^{t} \sim \mathcal{N}(0,I)}\l[ \frac{v_2^2}{\|v\|^2}\r] = \dots = \E_{v^{t} \sim \mathcal{N}(0,I)}\l[ \frac{v_d^2}{\|v\|^2}\r].
    \end{eqnarray*}
    Then, 
    \begin{eqnarray*}
        d \cdot \E_{v^{t} \sim \mathcal{N}(0,I)}\l[ \frac{v_1^2}{\|v\|^2}\r] 
        = \E_{v^{t} \sim \mathcal{N}(0,I)}\l[ \frac{v_1^2 + v_2^2 + \dots +v_d^2}{\|v\|^2}\r]
        = \E_{v^{t} \sim \mathcal{N}(0,I)}\l[ 1 \r] = 1.
    \end{eqnarray*}
    And we reach:
    \begin{eqnarray}\label{eq:ZOSGD_lem_}
        \E_{v^{t} \sim \mathcal{N}(0,I)}\l[ \frac{v_1^2}{\|v\|^2}\r] = \frac{1}{d}
    \end{eqnarray}
    Substitution of \eqref{eq:ZOSGD_lem_}, \eqref{eq:zosgd_1} into Lemma \ref{lem:descent} and rearranging the terms finishes the proof. 
\end{proof}

\begin{lemma}\label{lem:L_asc}
    Let function $f$ is $L$-smooth, i.e.
    \begin{eqnarray*}
        \|f(x) - f(y)\| \leq \|x - y\|,  \quad \forall x,y\in\mathbb{R}^d,
    \end{eqnarray*}
    then it holds
    \begin{eqnarray*}
        f(y) \geq f(x) + \langle\nabla f(x), y-x\rangle - \frac{L}{2}\|x-y\|^2, \quad \forall x,y\in\mathbb{R}^d.
    \end{eqnarray*}
\end{lemma}

\begin{proof}
The claim follows from integrating the directional derivative of $f$ along the segment connecting $x$ and $y$ and then applying the Lipschitz continuity of the gradient.
\begin{eqnarray}
f(y) & = & f(x) + \int_0^1 \langle \nabla f(x + \tau (y - x)), y - x \rangle d\tau \notag\\
& = & f(x) + \langle \nabla f(x), y - x \rangle \notag\\
&&+~ \int_0^1 \langle \nabla f(x + \tau (y - x)) - \nabla f(x), y - x \rangle d\tau \notag\\
& \geq & f(x) + \langle \nabla f(x), y - x \rangle \notag\\
&&-~ \int_0^1 \| \nabla f(x + \tau (y - x)) - \nabla f(x) \| \| y - x \| d\tau \notag\\
& \geq & f(x) + \langle \nabla f(x), y - x \rangle - \int_0^1 L \tau \| y - x \|^2 d\tau \notag\\
& = & f(x) + \langle \nabla f(x), y - x \rangle - \frac{L}{2} \| y - x \|^2 . \notag
\end{eqnarray}
\end{proof}

\begin{theoremrepeat}{th:main}
     Let Assumption \ref{ass:lip} hold and $d \geq 16$. Then for iterations of the Algorithm \ref{alg:theor1} with $\varepsilon \leq \frac{d^{-\frac{3}{2}}  \cos(\beta^t)(1-\cos(\beta^t)) \|\mu^t\|}{480}$, $\gamma_\mu^t \leq \frac{\|\mu^t\|^2}{640},\gamma_x^t \leq \frac{\cos^2(\beta^t)(1-\cos(\beta^t))^2}{3\cdot2^{12} \cdot 5^2 L}$ the following inequality holds:
\end{theoremrepeat}
    \begin{eqnarray}\label{eq:main_th_Statement}
     \E\l[ C^{t+1}\r | \mathcal{F}^t]
      &=&
      \E\l[ C^{t}\r | \mathcal{F}^{t-1}]
     +\gamma^t_\mu
     \l(\frac{\cos(\beta^t)(1-\cos(\beta^t))}{24 \|\mu^t\|}\r)^2,
\end{eqnarray}
\textit{where $\beta^t$ is the angle between $\mu^t$ and $\nabla f (x^t)$.}

Due to the non-standard challenges and the resulting proof technique, we provide high-level sketch before the proof.

\paragraph{Proof Sketch of Theorem \ref{th:main}.}
The proof of Theorem \ref{th:main} is deferred to Section \ref{sec:mis_proofs}. It tracks the joint evolution of two quantities: the gradient $\nabla f(x^t)$ and the mean vector $\mu^t$. The main steps are as follows:

    \textbf{(i)}
    We first reduce the gradient dynamics using Assumption \ref{ass:lip}, which yields a relation of the form
    \begin{eqnarray}\label{eq:sk_1}
        \E\l[ C^{t+1}\r | \mathcal{F}^t]  
    \geq -2a^t 
    + F_{\overline{\nabla f(x^t)}}(\mu^{t+1}),
    \end{eqnarray}
    with $F_{z}(\mu) = \E_{u\sim \mathcal{N}(\mu, \varepsilon^2 I)} [\langle z,\overline{u}\rangle^2]$ and $a^t = 4 \frac{L\gamma_x^t}{2}\sqrt{1 - \l(\frac{L\gamma_x^t}{2}\r)^2}$. We note that $\E\l[ C^{t}\r | \mathcal{F}^{t-1}] = F_{\overline{\nabla f(x^t)}}(\mu^{t})$. Then, we need to reduce $\mu^{t+1}$ to $\mu^t$ in the RHS of \eqref{eq:sk_1} to obtain a recursion on expected gradient alignment. 

    \textbf{(ii)} The core objective is to prove strict monotonic growth of expected gradient alignment $\E\l[ C^{t+1}\r | \mathcal{F}^t] $. Achieving this requires carefully isolating a positive contribution that dominates the drift $a^t$ induced by changes in $\nabla f(x^t)$.
    Then, to enable a quantitative analysis, we establish smoothness of the function $F_{z}(\mu)$ with respect to $\mu^t$ via a uniform bound on the Hessian norm. This yields an inequality of the form
    \begin{eqnarray}\label{eq:sk_2}
         F_{\overline{\nabla f(x^t)}}(\mu^{t+1}) \geq
         F_{\overline{\nabla f(x^t)}}(\mu^{t})
         +\l( \gamma^t_\mu - \frac{\tilde{L}{\gamma^t_\mu}^2}{2}\r)\| \nabla F_{\overline{\nabla f(x^t)}}(\mu^{t}) \|^2.
    \end{eqnarray}

    \textbf{(iii)} Combining the above relations \eqref{eq:sk_1} and \eqref{eq:sk_2} yields a recursive inequality for $\E\l[ C^{t+1}\r | \mathcal{F}^t]$. To ensure a positive increment, it suffices to lower bound the contribution of the last term in \eqref{eq:sk_2}.
    To this end, we decompose the expectation over $u$ in $F_{z}(\mu)$ definition into integrals over an inner region $u \leq \frac{\|\mu^t\|}{2\varepsilon}$ and its complement.
    
     The integral over the outer region is controlled by choosing $\varepsilon\lesssim \|\mu\|d^{-\frac{1}{2}}$, allowing us to apply Gaussian concentration inequalities for Lipschitz functions \citep{10.1093/acprof:oso/9780199535255.001.0001}.
    
    The inner-region contribution is analysed via a Taylor expansion with a Lagrange remainder. Each term is treated separately.
    First, symmetry considerations eliminate zero order term.
    Second, the second-order term is controlled using the previously established Hessian bound.
    Third, the Laurent-Massart inequalities \citep{laurent2000adaptive} serve as a core ingredient for extracting the required positive increment from the first order term.

    Putting these estimates together with \eqref{eq:sk_1}, \eqref{eq:sk_2} and balancing the competing effects yields the final bound.

    We present the complete proof below.

\begin{proof}[\textbf{Proof of Theorem \ref{th:main}}]

    The evolution of $\mathbb{E}_t[C^t]$ depends on the dynamics of two quantities: the gradient $\nabla f(x^t)$ itself and the mean vector $\mu^t$. We begin with \textbf{STEP I} by examining the contribution of the gradient dynamics.

    Using Assumption \ref{ass:lip} we bound the $\|\overline{\nabla f(x^{t+1})} - \overline{\nabla f(x^{t})}\|$ in the following way.

By the $x^{t+1}$ update rule $\|x^{t+1} - x^t\| = \|\gamma_x^t g_x^t\| = \gamma_x^t\l\|\overline{v^t} \l\langle\overline{v^t},{\nabla f(x^t)}\r\rangle\r\| \leq \gamma_x^t \|\nabla f(x^t)\|$. Therefore, $\|{\nabla f(x^{t+1})} - {\nabla f(x^{t})}\| \leq L\gamma_x^t \|{\nabla f(x^{t})}\|$, which implies that the angle between ${\nabla f(x^{t})}$ and ${\nabla f(x^{t+1})}$ less or equal then 
\[2\arcsin\l(\frac{\|{\nabla f(x^{t+1})} - {\nabla f(x^{t})}\|}{2\|{\nabla f(x^{t})}\|}\r) \leq 2\arcsin\l(\frac{L\gamma_x^t \|{\nabla f(x^{t})}\|}{2\|{\nabla f(x^{t})}\|}\r) = 2\arcsin\l(\frac{L\gamma_x^t}{2}\r).\]
The angle between $\overline{\nabla f(x^{t})}$ and $\overline{\nabla f(x^{t+1})}$ is the same. It justifies 
\begin{eqnarray}\label{eq:grad_dif_bound}
    \|\overline{\nabla f(x^{t+1})} - \overline{\nabla f(x^{t})}\| &\leq& 2 \sin\l( 2\arcsin\l( \frac{L\gamma_x^t}{2} \r) \r) \notag \\
    &=& 4 \frac{L\gamma_x^t}{2}\sqrt{1 - \l(\frac{L\gamma_x^t}{2}\r)^2} =: a^t 
\end{eqnarray}

And now we proceed with the definition of $C^t$.
\begin{eqnarray}\label{eq:C_begin}
    &&\E_{v^{t+1} \sim \mathcal{N}(\mu^{t+1},\varepsilon^2 I)}\l[ C^{t+1}\r | \mathcal{F}^t] \notag \\
    &&~~ = \frac{1}{(2\pi\varepsilon^2)^{d/2}} \int_{\mathbb{R}^d}
    \l\langle \overline{\nabla f(x^{t+1})}, \overline{u}\r\rangle^2 e^{- \frac{\|u - \mu^{t+1}\|^2}{2\varepsilon^2}} \mathrm{d}u \notag \\
    &&~~ = \frac{1}{(2\pi\varepsilon^2)^{d/2}} \int_{\mathbb{R}^d}
    \l\langle\overline{\nabla f(x^{t})} + (\overline{\nabla f(x^{t+1})} - \overline{\nabla f(x^{t})}), \overline{u}\r\rangle^2 e^{- \frac{\|u - \mu^{t+1}\|^2}{2\varepsilon^2}} \mathrm{d}u \notag \\
    &&~~= \frac{1}{(2\pi\varepsilon^2)^{d/2}} \int_{\mathbb{R}^d}
    \bigg\{\l\langle \overline{\nabla f(x^{t})}, \overline{u}\r\rangle^2
    + \l\langle\overline{\nabla f(x^{t+1})} - \overline{\nabla f(x^{t})}, \overline{u}\r\rangle^2 \notag \\
    &&~~~~~~~+~ 2\l\langle\overline{\nabla f(x^{t+1})} - \overline{\nabla f(x^{t})}, \overline{u}\r\rangle
    \l\langle\overline{\nabla f(x^{t})}, \overline{u}\r\rangle
    \bigg\} e^{- \frac{\|u - \mu^{t+1}\|^2}{2\varepsilon^2}} \mathrm{d}u \notag \\
    &&~~\geq \frac{1}{(2\pi\varepsilon^2)^{d/2}} \int_{\mathbb{R}^d}
    \bigg\{\l\langle \overline{\nabla f(x^{t})}, \overline{u}\r\rangle^2 \notag \\
    &&~~~~~~~+~ 2\l\langle \overline{\nabla f(x^{t+1})} - \overline{\nabla f(x^{t})}, \overline{u}\r\rangle
    \l\langle\overline{\nabla f(x^{t})}, \overline{u}\r\rangle
    \bigg\} e^{- \frac{\|u - \mu^{t+1}\|^2}{2\varepsilon^2}} \mathrm{d}u
\end{eqnarray}

Using Cauchy–Schwarz inequality and the bound \eqref{eq:grad_dif_bound} we derive
\begin{eqnarray}\label{eq:cross_bound}
    \l\langle\overline{\nabla f(x^{t+1})} - \overline{\nabla f(x^{t})}, \overline{u}\r\rangle
    \l\langle\overline{\nabla f(x^{t})}, \overline{u}\r\rangle
    \leq a^t.
\end{eqnarray}

Utilizing \eqref{eq:cross_bound}, we continue with derivation \eqref{eq:C_begin} by introducing a constraint on the integral of the last term with respect to the probability measure.
\begin{eqnarray}\label{eq:C_no_grad_dynamics}
    &&\E_{v^{t+1} \sim \mathcal{N}(\mu^{t+1},\varepsilon^2 I)}\l[ C^{t+1}\r | \mathcal{F}^t]  
    \geq -2a^t + \frac{1}{(2\pi\varepsilon^2)^{d/2}} 
    \int_{\mathbb{R}^d}
    \l\langle \overline{\nabla f(x^{t})}, \overline{u}\r\rangle^2 e^{- \frac{\|u - \mu^{t+1}\|^2}{2\varepsilon^2}} \mathrm{d}u. \notag \\
\end{eqnarray}

Thus, the gradient dynamics have been eliminated.

The next \textbf{STEP II} dedicated to the mean vector dynamics.

We start by introducing the function $F_{z}(\mu) = \E_{u\sim \mathcal{N}(\mu, \varepsilon^2 I)} [\langle z,\overline{u}\rangle^2]$ for every pair $\mu, z \in \mathbb{R}^d$. With this notation \eqref{eq:C_no_grad_dynamics} and the definition of $\E_{u \sim \mathcal{N}(\mu^{t+1},\varepsilon^2 I)}\l[ C^{t+1}\r | \mathcal{F}^t]$ can be reformulated as
\begin{eqnarray}\label{eq:C_F_dynamics}
    && \E_{v^{t+1} \sim \mathcal{N}(\mu^{t+1},\varepsilon^2 I)}\l[ C^{t+1}\r | \mathcal{F}^t]  
     \geq -2a^t + F_{\overline{\nabla f(x^t)}}(\mu^{t+1}),
\end{eqnarray}
\begin{eqnarray}\label{eq:C_F_def}
    && \E_{v^{t} \sim \mathcal{N}(\mu^{t},\varepsilon^2 I)}\l[ C^{t}\r | \mathcal{F}^{t-1}] = F_{\overline{\nabla f(x^t)}}(\mu^{t}).
\end{eqnarray}

Next, we observe that Lines \ref{line:g_v} and \ref{line:mu_update} of Algorithm \ref{alg:theor1} correspond to performing a step of gradient ascent with respect to $\mu$ on the function $F_{\overline{\nabla f(x^t)}}(\mu)$. Along with \eqref{eq:C_F_dynamics} and \eqref{eq:C_F_def} it defines the $\E_t\l[ C^t\r]$ dynamics.

Firstly, we introduce the kernel and change $C^t$ notation to reflect implicit dependence on arguments:
\begin{eqnarray}
    &&\varphi(x) = \frac{1}{(2\pi\varepsilon^2)^{d/2}} \exp\l(\frac{-\|x\|^2}{2\varepsilon^2}\r), \label{eq:phi_def} \\
    &&\psi_a(x) = (\overline{a},\overline{x})^2 = \frac{(a,x)^2}{\|x\|^2}.
\end{eqnarray}

Then we can reformulate $F_{\overline{\nabla f(x^t)}}(\mu)$ as follows
\begin{eqnarray}
    F_{a}(\mu) = \int_{\mathbb{R}^d} \psi_a (u) \varphi (\mu - u) \mathrm{d} u.
\end{eqnarray}

After that we start the Hessian computation using the general Stokes’ theorem. 
\begin{eqnarray}
     \nabla_\mu F_{a}(\mu) &=& \int_{\mathbb{R}^d} \psi_a (u) \nabla_\mu \varphi (\mu - u) \mathrm{d} u \notag\\
     &=& \int_{\mathbb{R}^d}  \psi_a (u) (-\nabla_u\varphi (\mu - u)) \mathrm{d} u \notag\\ 
    &=& - \int_{\mathbb{R}^d} \nabla_u (\psi_a (u) \varphi (\mu - u)) \mathrm{d} u + \int_{\mathbb{R}^d} \nabla_u (\psi_a (u)) \varphi (\mu - u) \mathrm{d} u \notag\\
    &=& \int_{\mathbb{R}^d} \nabla_u (\psi_a (u)) \varphi (\mu - u) \mathrm{d} u.
\end{eqnarray}

Following same reasoning we can derive
\begin{eqnarray}\label{eq:hess_reformulation_psi}
    \nabla_\mu^2 F_{a}(\mu) = \int_{\mathbb{R}^d} \nabla_u^2 \psi_a (u) \varphi (\mu - u) \mathrm{d} u = \E_{u\sim \mathcal{N}(\mu, \varepsilon^2 I)} \l[\nabla_u^2 \psi_a (u) \r].
\end{eqnarray}

Straightforward computation provides
\begin{eqnarray}
    \nabla_u^2 \psi_a(u) &=& \frac{1}{\|u\|^8} \big( 
    \|u\|^4 (2aa^\top \|u\|^2 +4au^\top\langle a,u\rangle - 2\langle a,u\rangle^2I \notag \\
     &&- 4ua^\top\langle a,u\rangle)
    - 4\|u\|^2 \l(2a\|u\|^2\langle a,u\rangle - 2\langle a,u\rangle^2 u\r)u^\top 
    \big).
\end{eqnarray}

It implies for arbitrary $a$ so that $\|a\|=1$ the following bound holds.

\begin{eqnarray}\label{eq:psi_hess_norm_bound}
    \|\nabla_u^2 \psi_a(u)\| \leq 20\frac{\|a\|^2}{\|u\|^2} = \frac{20}{\|u\|^2}.
\end{eqnarray}

Using \eqref{eq:hess_reformulation_psi}, \eqref{eq:psi_hess_norm_bound} and convexity of $\|\cdot\|$ we reach the bound 
\begin{eqnarray}\label{eq:hess_bound}
    \|\nabla_\mu^2 F_{a}(\mu)\| \leq 20 \E_{u\sim \mathcal{N}(\mu, \varepsilon^2 I)} \l[ \frac{1}{\|u\|^2} \r].
\end{eqnarray}

Assumption $d>3$ ensures the convergence of the integral in \eqref{eq:hess_bound} and provides the bound of the Hessian norm, thereby asserting the $\tilde{L}$-smoothness of the objective function $F_{a}(\mu)$ with $\tilde{L} \leq A \E_{u\sim \mathcal{N}(\mu, \varepsilon^2 I)} \l[ \frac{1}{\|u\|^2} \r]$. 

\begin{eqnarray}
    \E_{u\sim \mathcal{N}(\mu^t, \varepsilon^2 I)} \l[ \frac{1}{\|u\|^2} \r] 
    &=&
    \frac{1}{(2\pi\varepsilon^2)^{\frac{d}{2}}} \int_{\mathbb{R}^d}\frac{1}{\|u\|^2}e^{- \frac{\|u-\mu^t\|^2}{2\varepsilon^2}} \mathrm{d}u \notag \\
    &=&
    \frac{1}{(2\pi\varepsilon^2)^{\frac{d}{2}}} 
    \left[
    \int_{\|u\| \leq \frac{\|\mu^t\|}{4}}\frac{1}{\|u\|^2}e^{- \frac{\|u-\mu^t\|^2}{2\varepsilon^2}}\mathrm{d}u
    +
    \int_{\|u\| > \frac{\|\mu^t\|}{4}}\frac{1}{\|u\|^2}e^{- \frac{\|u-\mu^t\|^2}{2\varepsilon^2}}\mathrm{d}u
    \right]\notag \\
    &\leq&
    \frac{1}{(2\pi\varepsilon^2)^{\frac{d}{2}}} 
    \int_{\|u\| \leq \frac{\|\mu^t\|}{4}}\frac{1}{\|u\|^2}e^{- \frac{\|u\|^2+\|\mu^t\|^2 - 2\langle u,\mu^t \rangle}{2\varepsilon^2}}\mathrm{d}u
    +
    \frac{16}{\|\mu^t\|^2}
    \notag \\
    &\leq&
    \frac{1}{(2\pi\varepsilon^2)^{\frac{d}{2}}} 
    \int_{\|u\| \leq \frac{\|\mu^t\|}{4}}\frac{1}{\|u\|^2}e^{- \frac{\|u\|^2+\frac{\|\mu^t\|^2}{2}}{2\varepsilon^2}}\mathrm{d}u
    +
    \frac{16}{\|\mu^t\|^2}
    \notag \\
    &\leq&
    e^{-\frac{\|\mu^t\|^2}{4\varepsilon^2}}
    \frac{1}{(2\pi\varepsilon^2)^{\frac{d}{2}}} 
    \int_{\mathbb{R}^d}\frac{1}{\|u\|^2}e^{- \frac{\|u\|^2}{2\varepsilon^2}}\mathrm{d}u
    +
    \frac{16}{\|\mu^t\|^2} \notag\\
    &=&
    \frac{e^{-\frac{\|\mu^t\|^2}{4\varepsilon^2}}}{\varepsilon^2 (d-2)}
    +
    \frac{16}{\|\mu^t\|^2}.
    \notag 
\end{eqnarray}

In the last transition we used classical result for $\chi^2$ distribution $r$ moment $2^r\frac{\Gamma\l(\frac{d}{2}+r\r)}{\Gamma\l( \frac{d}{2}\r)}$. Then, $\tilde{L}\leq \frac{20e^{-\frac{\|\mu^t\|^2}{4\varepsilon^2}}}{\varepsilon^2 (d-2)} + \frac{16 \cdot 20}{\|\mu^t\|^2}$.

Applying Lemma \ref{lem:L_asc} to function $F$ taking into account implicit step formulation $\mu^{t+1} = \mu^{t} + \gamma^t_\mu\nabla F_{\overline{\nabla f(x^t)}}(\mu^t)$ we get 
\begin{eqnarray}\label{eq:F_dynamics_general}
    F_{\overline{\nabla f(x^t)}}(\mu^{t+1}) &\geq&  F_{\overline{\nabla f(x^t)}}(\mu^{t}) +  ( \nabla F_{\overline{\nabla f(x^t)}}(\mu^{t}), \mu^{t+1} - \mu^{t} ) - \frac{\tilde{L}}{2} \| \mu^{t+1} - \mu^{t} \|^2 \notag \\
    &=& F_{\overline{\nabla f(x^t)}}(\mu^{t}) +\l( \gamma^t_\mu - \frac{\tilde{L}{\gamma^t_\mu}^2}{2}\r)\| \nabla F_{\overline{\nabla f(x^t)}}(\mu^{t}) \|^2
\end{eqnarray}

By combining \eqref{eq:C_F_dynamics}, \eqref{eq:C_F_def} and \eqref{eq:F_dynamics_general} we derive

\begin{eqnarray}\label{eq:C_F_pre_full_dynamics}
     \E_{v^{t+1} \sim \mathcal{N}(\mu^{t+1},\varepsilon^2 I)}\l[ C^{t+1}\r | \mathcal{F}^t]  
     &\geq& -2a^t + F_{\overline{\nabla f(x^t)}}(\mu^{t}) +\l( \gamma^t_\mu - \frac{\tilde{L}{\gamma^t_\mu}^2}{2}\r)\| \nabla F_{\overline{\nabla f(x^t)}}(\mu^{t}) \|^2 \notag \\
     &\overset{\eqref{eq:C_F_def}}{=}& \E_{v^{t} \sim \mathcal{N}(\mu^{t},\varepsilon^2 I)}\l[ C^{t}\r | \mathcal{F}^{t-1}]   \notag \\
     &&-~ 2a^t +\l( \gamma^t_\mu - \frac{\tilde{L}{\gamma^t_\mu}^2}{2}\r)\| \nabla F_{\overline{\nabla f(x^t)}}(\mu^{t}) \|^2
\end{eqnarray}

At the next step we work on the lower bound of $\|\nabla F_a(\mu)\|$ in particular $\mu$-space regions.
\begin{eqnarray} \label{eq:2_int_decomposition}
    \nabla F_a(\mu) &=&  \frac{1}{(2\pi\varepsilon^2)^{d/2}} \int_{\mathbb{R}^d}
    \frac{\langle a,  u\rangle^2}{\|u\|^2} \frac{u-\mu}{\varepsilon^2} 
    e^{- \frac{\|u - \mu\|^2}{2\varepsilon^2}} \mathrm{d}u \notag \\
    &=& \frac{1}{\varepsilon(2\pi)^{d/2}} \int_{\mathbb{R}^d}
    \frac{\langle a,  \varepsilon v +\mu\rangle^2}{\|\varepsilon v +\mu\|^2} v 
    e^{- \frac{\|v\|^2}{2}} \mathrm{d}v \notag \\
    &=& \frac{1}{\varepsilon(2\pi)^{d/2}} \int_{\|v\|<{\frac{\|\mu\|}{2\varepsilon}}}
    \frac{\langle a,  \varepsilon v +\mu\rangle^2}{\|\varepsilon v +\mu\|^2} v 
    e^{- \frac{\|v\|^2}{2}} \mathrm{d}v\notag \\
    &&+~ \frac{1}{\varepsilon(2\pi)^{d/2}} \int_{\|v\|>{\frac{\|\mu\|}{2\varepsilon}}}
    \frac{\langle a,  \varepsilon v +\mu\rangle^2}{\|\varepsilon v +\mu\|^2} v 
    e^{- \frac{\|v\|^2}{2}} \mathrm{d}v.
\end{eqnarray}

We proceed with evaluation of the second term of \ref{eq:2_int_decomposition}.
\begin{eqnarray}\label{eq:concentraton_begining}
    && \l\|\frac{1}{(2\pi)^{d/2}} \int_{\|v\|>{\frac{\|\mu\|}{2\varepsilon}}}
    \frac{\langle a,  \varepsilon v +\mu\rangle^2}{\|\varepsilon v +\mu\|^2} v 
    e^{- \frac{\|v\|^2}{2}} \mathrm{d}v \r\| \notag \\
    && \leq \frac{1}{(2\pi)^{d/2}} \int_{\|v\|>{\frac{\|\mu\|}{2\varepsilon}}}
    \frac{\langle a,  \varepsilon v +\mu\rangle^2}{\|\varepsilon v +\mu\|^2} \|v\| 
    e^{- \frac{\|v\|^2}{2}} \mathrm{d}v  \notag \\
    && \leq \frac{1}{(2\pi)^{d/2}} \int_{\|v\|>{\frac{\|\mu\|}{2\varepsilon}}}
     \|v\| 
    e^{- \frac{\|v\|^2}{2}} \mathrm{d}v 
\end{eqnarray}

Then we use the fact that for $d$-dimensional $\xi \sim \mathcal{N}(0,I)$ holds the further bound \citep{chandrasekaran2012convex}:
\begin{eqnarray}\label{eq:norm_expectation_bound}
    \E\l[ \|\xi\|\r] = \sqrt{2}\frac{\Gamma\l( \frac{d+1}{2}\r)}{\Gamma\l( \frac{d}{2}\r)} \leq \sqrt{d}.
\end{eqnarray}

After that we use Cauchy–Schwarz inequality in $L_2(\mathbb{R}^d)$ space with Gaussian measure to derive
\begin{eqnarray}\label{eq:CS_L2}
    \E\l[ I_{\l\{\|\xi\|>t\r\}}
     \|\xi\| \r]
    &\leq& \l( \E\l[ \|\xi\|^2\r] \r)^{\frac{1}{2}} \l( \E\l[  I_{\l\{\|\xi\|>t\r\}}^2\r] \r)^{\frac{1}{2}} \notag \\
    &=& \sqrt{d} \l( \Pr\l( \|\xi\|>t\r) \r)^{\frac{1}{2}}.
\end{eqnarray}

Gaussian measure concentration inequality \citep{10.1093/acprof:oso/9780199535255.001.0001} for $L =1$ Lipchitz function $\|\cdot\|$ yields
\begin{eqnarray}
    \Pr\l( \|\xi\| -\E[\|\xi\|] > t\r) \leq  e^{-\frac{t^2}{2}}. \notag
\end{eqnarray}

Combining it with \eqref{eq:norm_expectation_bound} we get $\Pr\l( \|\xi\| -\sqrt d > t\r) \leq  e^{-\frac{t^2}{2}}$ or
\begin{eqnarray}\label{eq:measure_bound}
    \Pr\l( \|\xi\| > \tau\r) \leq  e^{-\frac{(\tau-\sqrt{d})^2}{2}}, ~ \tau>\sqrt{d}. 
\end{eqnarray}

Substitution of \eqref{eq:measure_bound} into \eqref{eq:CS_L2} provides
\begin{eqnarray}\label{eq:final_measure_bound}
    \E\l[ I_{\l\{\|\xi\|>\tau\r\}}
     \|\xi\| \r]
    &=& \sqrt{d} e^{-\frac{(\tau - \sqrt d)^2}{4}}, ~ \tau > \sqrt d.
\end{eqnarray}

Finally, using \eqref{eq:final_measure_bound} we transform \eqref{eq:concentraton_begining} into
\begin{eqnarray}\label{eq:outball_bound}
     \l\|\frac{1}{\varepsilon(2\pi)^{d/2}} \int_{\|v\|>{\frac{\|\mu\|}{2\varepsilon}}}
    \frac{\langle a,  \varepsilon v +\mu\rangle^2}{\|\varepsilon v +\mu\|^2} v 
    e^{- \frac{v^2}{2}} \mathrm{d}v \r\| 
    \leq \frac{1}{\varepsilon}
    \sqrt{d} e^{-\frac{\l(\frac{\|\mu\|}{2\varepsilon} - \sqrt d \r)^2}{4}}, ~ \frac{\|\mu\|}{2\varepsilon} > \sqrt d.
\end{eqnarray}

Then we turn to the evaluation of the first term in \eqref{eq:2_int_decomposition}.

\begin{eqnarray}\label{eq:for_tailor_preparation}
    &&\frac{1}{\varepsilon(2\pi)^{d/2}} \int_{\|v\|<{\frac{\|\mu\|}{2\varepsilon}}}
    \frac{\langle a,  \varepsilon v +\mu\rangle^2}{\|\varepsilon v +\mu\|^2} v 
    e^{- \frac{\|v\|^2}{2}} \mathrm{d}v \notag\\
    &&= \frac{1}{\varepsilon(2\pi)^{d/2}} \int_{\|v\|<{\frac{\|\mu\|}{2\varepsilon}}}
    \frac{\langle a,  \varepsilon v +\mu\rangle^2}{\|\varepsilon v +\mu\|^2} v 
    e^{- \frac{\|v\|^2}{2}} \mathrm{d}v \notag \\
    &&= \frac{1}{\varepsilon(2\pi)^{d/2}} \int_{\|v\|<{\frac{\|\mu\|}{2\varepsilon}}}
    \psi(\varepsilon v +\mu) v 
    e^{- \frac{\|v\|^2}{2}} \mathrm{d}v 
\end{eqnarray}

We proceed by using the Taylor expansion with the remainder term in Lagrange form.

\begin{eqnarray}\label{eq:tailor}
    \psi_a(\varepsilon v +\mu) v  = \psi_a(\mu)v + \langle \nabla \psi_a(\mu), \varepsilon v \rangle v + \frac{1}{2}\l[\l(\varepsilon v^\top\r) H_\psi(\mu + \varepsilon v') \l(\varepsilon v\r)\r]v,
\end{eqnarray}
for $v' = \alpha v, 0\leq \alpha \leq 1$.

Symmetry considerations provides first term vanishing after integration over $\l\{ \|v\|<{\frac{\|\mu\|}{2\varepsilon}} \r\}$. Symmetry also simplifies the second term of \eqref{eq:tailor} the in following way

\begin{eqnarray}\label{tailor:second_term_begining}
    &&\frac{1}{\varepsilon(2\pi)^{d/2}}\l( \int_{\|v\|<{\frac{\|\mu\|}{2\varepsilon}}} \langle \nabla\psi_a(\mu),\varepsilon v\rangle v e^{- \frac{\|v\|^2}{2}} \mathrm{d}v \r)_i \notag \\
    &&=  \frac{1}{(2\pi)^{d/2}}\int_{\|v\|<{\frac{\|\mu\|}{2\varepsilon}}} \sum_{j=1}^d \nabla\psi_a(\mu)_j v_j v_i e^{- \frac{\|v\|^2}{2}} \mathrm{d}v \notag \\
    &&=  \frac{1}{(2\pi)^{d/2}}\int_{\|v\|<{\frac{\|\mu\|}{2\varepsilon}}} \nabla\psi_a(\mu)_i v_i v_i e^{- \frac{\|v\|^2}{2}} \mathrm{d}v \notag \\
    &&= \frac{\nabla\psi_a(\mu)_i}{(2\pi)^{d/2}} \int_{\|v\|<{\frac{\|\mu\|}{2\varepsilon}}}  v_1^2 e^{- \frac{\|v\|^2}{2}} \mathrm{d}v \notag \\
    &&= \frac{\nabla\psi_a(\mu)_i}{(2\pi)^{d/2}} \frac{1}{d} \int_{\|v\|<{\frac{\|\mu\|}{2\varepsilon}}}  \|v\|^2 e^{- \frac{\|v\|^2}{2}} \mathrm{d}v .
\end{eqnarray}

We apply Laurent-Massart inequalities \citep{laurent2000adaptive} to construct the lower bound of the integral in \eqref{tailor:second_term_begining}.
\begin{eqnarray}\label{eq:Laurent-Massart}
    &&\Pr\l( \|v\|^2 -d \leq 2\sqrt{dx} + 2x \r) \leq e^{-x}, \notag \\
    &&\Pr\l( d - \|v\|^2 \geq 2\sqrt{dx}\r) \leq e^{-x}.
\end{eqnarray}

Assuming $x=1$ and $\frac{\|\mu\|}{2\varepsilon} \geq \sqrt{d + 2\sqrt{d} + 2}$ we get

\begin{eqnarray}\label{eq:Laurent-Massart_application}
    && \frac{1}{(2\pi)^{d/2}} \int_{\|v\|<{\frac{\|\mu\|}{2\varepsilon}}}  \|v\|^2 e^{- \frac{\|v\|^2}{2}} \mathrm{d}v \notag \\
    && \geq \frac{1}{(2\pi)^{d/2}} \int_{
    d - 2\sqrt{d}<
    \|v\|^2
    <d + 2\sqrt{d} + 2}
    \|v\|^2 e^{- \frac{\|v\|^2}{2}} \mathrm{d}v \notag \\
    && \geq (d - 2\sqrt{d}) \Pr\l(d - 2\sqrt{d}<
    \|v\|^2
    <d + 2\sqrt{d} + 2\r) \notag \\
    && \geq (d - 2\sqrt{d})(1-2e^{-1})\notag \\
    && \geq \frac{(d - 2\sqrt{d})}{4}
\end{eqnarray}

Combining \eqref{eq:Laurent-Massart_application} and \eqref{tailor:second_term_begining} we derive

\begin{eqnarray}\label{eq:positive_term_final_estim}
    \l\| \frac{1}{(2\pi)^{d/2}} \int_{\|v\|<{\frac{\|\mu\|}{2\varepsilon}}} \langle \nabla\psi_a(\mu),v\rangle v e^{- \frac{v^2}{2}} \mathrm{d}v \r\| 
    &\geq& \frac{\|\nabla\psi_a(\mu)\|}{ d} \frac{(d-2\sqrt d)}{4} \notag \\ 
    &=& \|\nabla\psi_a(\mu)\| \l( \frac{1}{4} - \frac{1}{2\sqrt d} \r) \notag \\
    &\geq& \frac{\|\nabla\psi_a(\mu)\|}{8}
\end{eqnarray}

To justify the last transition we narrow our analysis to the case $d \geq 16$.

Then we examine the last term of Taylor Expansion \eqref{tailor:second_term_begining} using \eqref{eq:psi_hess_norm_bound}.

\begin{eqnarray}
     \l\|\frac{1}{2}\l[\l(\varepsilon v^\top\r) H_\psi(\mu + \varepsilon v')  \l(\varepsilon v\r)\r]v \r\| 
     &\leq& \frac{\varepsilon^2}{2} \l \| v  \r\| ^3 \l \| H_\psi(\mu + \varepsilon v')
     \r\|  \notag \\
     &\overset{\eqref{eq:psi_hess_norm_bound}}{\leq}&  \frac{\varepsilon^2}{2} \l \| v  \r\| ^3 \frac{20}{\|\mu + \varepsilon v'\|^2}.
       \notag
\end{eqnarray}

Using norm convexity we conclude
\begin{eqnarray}\label{eq:hess_int_bound_d3}
    &&\l\| \frac{1}{\varepsilon(2\pi)^{d/2}} \int_{\|v\|<{\frac{\|\mu\|}{2\varepsilon}}} \frac{1}{2}\l[\l(\varepsilon v^\top\r) H_\psi(\mu + \varepsilon v')  \l(\varepsilon v\r)\r]v e^{- \frac{\|v\|^2}{2}}\mathrm{d}v \r\| \notag \\
    &&\leq  \frac{1}{\varepsilon(2\pi)^{d/2}} \int_{\|v\|<{\frac{\|\mu\|}{2\varepsilon}}} \l\| \frac{1}{2}\l[\l(\varepsilon v^\top\r) H_\psi(\mu + \varepsilon v')  \l(\varepsilon v\r)\r]v\r\| e^{- \frac{\|v\|^2}{2}}\mathrm{d}v  \notag \\
    &&\leq  \frac{1}{\varepsilon(2\pi)^{d/2}} \int_{\|v\|<{\frac{\|\mu\|}{2\varepsilon}}} \frac{\varepsilon^2}{2} \l \| v  \r\| ^3 \|\frac{20}{\|\mu + \varepsilon v'\|^2} e^{- \frac{\|v\|^2}{2}}\mathrm{d}v  \notag \\
    &&\leq  \frac{1}{\varepsilon(2\pi)^{d/2}} \int_{\|v\|<{\frac{\|\mu\|}{2\varepsilon}}} \frac{\varepsilon^2}{2} \l \| v  \r\| ^3 \|\frac{80}{\|\mu\|^2} e^{- \frac{\|v\|^2}{2}}\mathrm{d}v  \notag \\
    &&\leq  \frac{40\varepsilon}{\|\mu\|^2(2\pi)^{d/2}} \int_{\|v\|<{\frac{\|\mu\|}{2\varepsilon}}}   \l \| v  \r\| ^3e^{- \frac{\|v\|^2}{2}} \mathrm{d}v  \notag \\
    &&\leq  \frac{40\varepsilon}{\|\mu\|^2(2\pi)^{d/2}} \int_{\mathbb{R}^d}   \l \| v  \r\| ^3e^{- \frac{\|v\|^2}{2}} \mathrm{d}v  \notag \\
    &&\leq  \frac{40\varepsilon}{\|\mu\|^2} d^{\frac{3}{2}}.
\end{eqnarray}

The last transition utilizes Cauchy-Schwarz inequality and already mentioned values of $\E[\|v\|] \leq \sqrt{d}$ and $\E[\|v\|^2] = d$. Finally, using triangle inequality as $\|a+b\| \geq \|a\|-\|b\|$ we combine \eqref{eq:2_int_decomposition}, \eqref{eq:outball_bound}, \eqref{eq:positive_term_final_estim} and \eqref{eq:hess_int_bound_d3} to reach

\begin{eqnarray}\label{eq:final_grad_lower_bound}
    \|\nabla F_a(\mu)\| &\geq& 
    \frac{\|\nabla\psi_a(\mu)\|}{8} - \frac{40\varepsilon}{\|\mu\|^2} d^{\frac{3}{2}}
    -\frac{\sqrt{d}}{\varepsilon}e^{-\frac{\l( \frac{\|\mu\|}{2 \varepsilon} - \sqrt{d}\r)^2}{4}}.
\end{eqnarray}

After that we can continue the recursion  \eqref{eq:C_F_pre_full_dynamics} applying \eqref{eq:final_grad_lower_bound}.

\begin{eqnarray}\label{eq:C_F_full_dynamics}
     \E_{v^{t+1} \sim \mathcal{N}(\mu^{t+1},\varepsilon I)}\l[ C^{t+1}\r | \mathcal{F}^t]
     &\geq&
     \E_{v^{t} \sim \mathcal{N}(\mu^{t},\varepsilon I)}\l[ C^{t}\r | \mathcal{F}^{t-1}]  - 2a^t\notag \\
     &&+
     \l( \gamma^t_\mu - \frac{\tilde{L}{\gamma^t_\mu}^2}{2}\r)
     \l( \frac{\|\nabla\psi_{\overline{\nabla f(x^t)}}(\mu^t)\|}{8} - \frac{40\varepsilon}{\|\mu^t\|^2} d^{\frac{3}{2}}
    -\frac{\sqrt{d}}{\varepsilon}e^{-\frac{\l( \frac{\|\mu^t\|}{2 \varepsilon} - \sqrt{d}\r)^2}{4}} \r)^2.
\end{eqnarray}

We beginning the evaluation of \eqref{eq:C_F_full_dynamics} by requiring

\begin{eqnarray}\label{cases}
    \begin{cases}
        & \frac{\|\nabla\psi_{\overline{\nabla f(x^t)}}(\mu^t)\|}{24} \geq \frac{40\varepsilon}{\|\mu^t\|^2} d^{\frac{3}{2}}, \\
        & \frac{\|\nabla\psi_{\overline{\nabla f(x^t)}}(\mu^t)\|}{24} \geq\frac{\sqrt{d}}{\varepsilon}e^{-\frac{\l( \frac{\|\mu^t\|}{2 \varepsilon} - \sqrt{d}\r)^2}{4}}, \\
        & \gamma^t_\mu \geq \tilde{L}{\gamma^t_\mu}^2.
    \end{cases}
\end{eqnarray}

To start the evaluation of conditions \eqref{cases} we introduce the angle  $\beta^t$ between the $\mu^t$ and $\nabla f (x^t)$.

\begin{eqnarray}\label{eq:angle_psi}
    \| \nabla \psi_{\overline{\nabla f(x^t)}}(\mu)\| &=& \l\| \frac{2\langle \nabla f(x^t),\mu\rangle}{\mu^4} \{\nabla f(x^t)\|\mu\|^2 - \mu \langle \nabla f(x^t),\mu\rangle \}\r\| \notag \\
    &&\geq \frac{2}{\|\mu\|^3}\cos(\beta^t)\l( \|\mu\|^2 - \|\mu\|^2\cos(\beta^t)\r) \notag \\
    &&\geq \frac{2}{\|\mu\|} \cos(\beta^t)(1-\cos(\beta^t)) \notag \\
\end{eqnarray}
Hence, choosing $\varepsilon \leq \frac{d^{-\frac{3}{2}}  \cos(\beta^t)(1-\cos(\beta^t)) \|\mu^t\|}{480}$ we satisfy the first condition. This choice of $\varepsilon$ also justifies second restriction in \eqref{cases}.
To satisfy the last inequality of \eqref{cases} we choose $\gamma_\mu^t \leq \frac{\|\mu^t\|^2}{640} {\leq} \frac{1}{\tilde{L}}$.
Then, \eqref{eq:C_F_full_dynamics} transforms into

\begin{eqnarray}\label{eq:C_F_full_dynamics1}
     \E_{v^{t+1} \sim \mathcal{N}(\mu^{t+1},\varepsilon I)}\l[ C^{t+1}\r | \mathcal{F}^t]
     &\geq&
     \E_{v^{t} \sim \mathcal{N}(\mu^{t},\varepsilon I)}\l[ C^{t}\r | \mathcal{F}^{t-1}]  - 2a^t
     +
     \frac{\gamma^t_\mu}{2}
      \frac{\|\nabla\psi(\mu^t)\|^2}{24^2} \notag\\
      &=&
      \E_{v^{t} \sim \mathcal{N}(\mu^{t},\varepsilon I)}\l[ C^{t}\r | \mathcal{F}^{t-1}]  - 2a^t
     +\gamma^t_\mu
     \frac{2\cos^2(\beta^t)(1-\cos(\beta^t))^2}{24^2 \|\mu^t\|^2}.
\end{eqnarray}

After that we require $4a^t \leq \frac{\cos^2(\beta^t)(1-\cos(\beta^t))^2}{3\cdot2^{9} \cdot 5^2}$. Recalling definition of $a^t = 4 \frac{L\gamma_x^t}{2}\sqrt{1 - \l(\frac{L\gamma_x^t}{2}\r)^2}$ we pick $\gamma_x^t \leq \frac{\cos^2(\beta^t)(1-\cos(\beta^t))^2}{3\cdot2^{12} \cdot 5^2L}$ to satisfy Stated restriction. Finally we reach

\begin{eqnarray}\label{eq:main_th_final}
     \E_{v^{t+1} \sim \mathcal{N}(\mu^{t+1},\varepsilon I)}\l[ C^{t+1}\r | \mathcal{F}^t]
      &=&
      \E_{v^{t} \sim \mathcal{N}(\mu^{t},\varepsilon I)}\l[ C^{t}\r | \mathcal{F}^{t-1}]
     +\gamma^t_\mu
     \l(\frac{\cos(\beta^t)(1-\cos(\beta^t))}{24 \|\mu^t\|}\r)^2.
\end{eqnarray}

\end{proof}

\begin{corollary}
    In the setting of Theorem \ref{th:main} for iterations of the Algorithm \ref{alg:theor1} with $\varepsilon = \frac{d^{-\frac{3}{2}}  \cos(\beta^t)(1-\cos(\beta^t)) M}{480}, \gamma_\mu^t = \frac{\|\mu^t\|^2}{640},\gamma_x^t = \frac{\cos^2(\beta^t)(1-\cos(\beta^t))^2}{3\cdot2^{12} \cdot 5^2 L}$ the following inequality holds:
\begin{eqnarray}\label{eq:main_th_with_eq}
     \E\l[ C^{t+1}\r | \mathcal{F}^t]
      &\geq&
      \E\l[ C^{t}\r | \mathcal{F}^{t-1}]
     +
     \frac{\cos^2(\beta^t)(1-\cos(\beta^t))^2}{3\cdot2^{10} \cdot 5^2}.
\end{eqnarray}
If additionally $\delta \leq \cos(\beta^t)\leq 1-\delta$ with $\delta \in \l(0;\frac{1}{2}\r)$ and $\varepsilon = \frac{d^{-\frac{3}{2}}  \delta M}{960},\gamma_x^t = \frac{\delta^2}{3\cdot2^{14} \cdot 5^2 L}$, then
\begin{eqnarray}\label{eq:main_th_with_delta}
     \E\l[ C^{t+1}\r | \mathcal{F}^t]
      &\geq&
      \E\l[ C^{t}\r | \mathcal{F}^{t-1}]
     +
     \frac{\delta^2}{3\cdot2^{12} \cdot 5^2}.
\end{eqnarray}
\end{corollary}
\begin{proof}
    The proof follows by direct substitution into the result of Theorem \ref{th:main}.
\end{proof}

\begin{lemmarepeat}{lem:multiitteration_eval}
    Let in the setting of Theorem \ref{th:main} $\delta \leq \cos(\beta^0)\leq 1-\delta$  with $\delta \in \l(0;\frac{1}{2}\r)$ and $\varepsilon = \frac{d^{-\frac{3}{2}}  \delta M}{960},\gamma_x^t = \frac{\delta^2}{3\cdot2^{14} \cdot 5^2 L}$. Then $\E\l[ C^t|\mathcal{F}^{t-1}\r]$ increase monotonically to the value $\cos\l(\frac{\delta}{32d} + L\gamma_x^t + \arccos(1-\delta)\r)^2 (1-e^{-1})$ and then never decrease below.
\end{lemmarepeat} 
\begin{proof}
    To justify the first part of Statement related to $\E\l[ C^t|\mathcal{F}^{t-1}\r]$ monotonically growth it sufficient to show, that the angle $\beta^t$ will decrease until $\cos(\beta^t)$ reaches $1-\delta$. We note that $C^t$ dynamic analysis closely related to the evaluation of the angle $\beta^t$. We proceed by assuming $\delta \leq \cos(\beta^t)\leq q-\delta$ considering 
    \begin{eqnarray}\label{eq:good_lem_begining}
        \|\overline{\mu^{t+1}}-\overline{\nabla f(x^{t+1})}\|
        &\overset{\eqref{eq:grad_dif_bound}}{\leq}&
        \|\overline{\mu^{t+1}}-\overline{\nabla f(x^{t})}\| + a^t \notag \\
        &=&
        \|\overline{\mu^{t}+\gamma^t_\mu \nabla_{\mu} \E_{t}\l[ C^{t}\r]\Big|_{\mu = \mu^{t}}}-\overline{\nabla f(x^{t})}\| + a^t \notag \\
        &=&
        \|\overline{\mu^{t}+\gamma^t_\mu F_{\overline{\nabla f(x^{t})}}(\mu^t)}-\overline{\nabla f(x^{t})}\| + a^t.
    \end{eqnarray}
Then we recall decomposition \eqref{eq:2_int_decomposition} and that one integral there contains \eqref{tailor:second_term_begining}. 
We remind that $ \nabla \psi(\mu) =  \frac{2\langle \nabla f(x^t),\mu\rangle}{\mu^4} \{\nabla f(x^t)\|\mu\|^2 - \mu \langle \nabla f(x^t),\mu\rangle \}$ which implies its collinearity to $\overline{\mu^{t}}-\overline{\nabla f(x^{t})}$. Using the same bounds as in \eqref{eq:final_grad_lower_bound} we continue derivation.
\begin{eqnarray}
    \|\overline{\mu^{t+1}}-\overline{\nabla f(x^{t+1})}\|
        &\leq& 
        \|\overline{\mu^{t}}-\overline{\nabla f(x^{t})}\|
        -\frac{1}{2 \|\mu^t\|} 
        \gamma^t_\mu 
     \l( \frac{\|\nabla\psi(\mu^t)\|}{8} - \frac{40\varepsilon}{\|\mu^t\|^2} d^{\frac{3}{2}}
    -\frac{\sqrt{d}}{\varepsilon}e^{-\frac{\l( \frac{\|\mu^t\|}{2 \varepsilon} - \sqrt{d}\r)^2}{4}} \r) 
    +a^t.
\end{eqnarray}
After that we note $\delta \leq \cos(\beta^t)\leq 1-\delta$ with $\delta \in \l(0;\frac{1}{2}\r)$ implies both $\cos(\beta^t)$ and $1-\cos(\beta^t)$ greater than or equal to $\delta$ and one of them greater then $\frac{1}{2}$, which yields $\cos(\beta^t)(1-\cos(\beta^0)) \geq \frac{\delta}{2}$. Using it along with a choice of parameters in Theorem \ref{th:main} (see transitions \eqref{eq:C_F_full_dynamics} to \eqref{eq:main_th_final}) we conclude 
\begin{eqnarray}
    \|\overline{\mu^{t+1}}-\overline{\nabla f(x^{t+1})}\|
        &\leq&     \|\overline{\mu^{t}}-\overline{\nabla f(x^{t})}\|,
\end{eqnarray}
Which means that once the condition $\delta \leq \cos(\beta^t) \leq 1 - \delta$ holds $\beta^t \geq \beta^{t+1}$. Hence, $\E\l[ C^t|\mathcal{F}^{t-1}\r]$ growth monotonically until $\cos(\beta^t)$ exceeds $(1-\delta)$. After that $\beta^t$ can not exceed $\delta_0 :=L\gamma_x^t+\arccos{(1-\delta)}$, since the direction of the gradient changes less then $2\arcsin\l(\frac{L\gamma_x^t}{2}\r)$ per step, as was shown in the beginning of the Theorem \ref{th:main} proof, and $\beta^t$ can not increase further if $\delta \leq \cos(\beta^t) \leq 1 - \delta$ as shown above.

Finally, we evaluate 
\begin{eqnarray}\label{eq:3_1}
    \E\l[ C^t|\mathcal{F}^{t-1}\r]
    &=&
    \frac{1}{(2\pi\varepsilon^2)^{d/2}} 
    \int_{\mathbb{R}^d}
    \l\langle \overline{\nabla f(x^{t})}, \overline{u}\r\rangle^2 e^{- \frac{\|u - \mu^{t}\|^2}{2\varepsilon^2}} \mathrm{d}u \notag\\
    &=&
    \frac{1}{(2\pi\varepsilon^2)^{d/2}} 
    \int_{\mathbb{R}^d}
    \l\langle \overline{\nabla f(x^{t})}, \overline{s - \mu^{t}}\r\rangle^2 e^{- \frac{\|s \|^2}{2\varepsilon^2}} \mathrm{d}s \notag\\
    &=&
    \frac{1}{(2\pi)^{d/2}} 
    \int_{\mathbb{R}^d}
    \l\langle \overline{\nabla f(x^{t})}, \overline{\varepsilon v - \mu^{t}}\r\rangle^2 e^{- \frac{\|v \|^2}{2}} \mathrm{d}v.
\end{eqnarray}
We apply Laurent-Massart inequality \citep{laurent2000adaptive} \eqref{eq:Laurent-Massart} with $x=1$ to justify $\Pr\l( d - \|v\|^2 \geq 2\sqrt{d}\r) \leq e^{-1}$. Then we can bound the integral in \eqref{eq:3_1}.
\begin{eqnarray}\label{eq:3_2}
    \frac{1}{(2\pi)^{d/2}} 
    \int_{\mathbb{R}^d}
    \l\langle \overline{\nabla f(x^{t})}, \overline{\varepsilon v - \mu^{t}}\r\rangle^2 e^{- \frac{\|v \|^2}{2}} \mathrm{d}u 
    &\geq& 
    \frac{1}{(2\pi)^{d/2}} 
    \int_{\| v\|^2 \leq 2\sqrt d}
    \l\langle \overline{\nabla f(x^{t})}, \overline{\varepsilon v - \mu^{t}}\r\rangle^2 e^{- \frac{\|v \|^2}{2}} \mathrm{d}u \notag \\
    &\geq& 
    \frac{1}{(2\pi)^{d/2}} 
    \int_{\| v\|^2 \leq 2\sqrt d}
    \cos\l(\frac{\delta}{32d} + L\gamma_x^t + \arccos(1-\delta)\r)^2 e^{- \frac{\|v \|^2}{2}} \mathrm{d}u \notag \\
    &\geq& 
    \cos\l(\frac{\delta}{32d} + L\gamma_x^t + \arccos(1-\delta)\r)^2 (1-e^{-1}).
\end{eqnarray}
It concludes the proof.
\end{proof}

\subsection{Initialization with no prior information} \label{sec:no_prior_init}

Lemma \ref{lem:multiitteration_eval} naturally leads to two complementary settings. The first is formalized in Lemma \ref{lem:arbitrary_convergence} and corresponds to initializing the policy without any prior knowledge of the target function $f$.

\begin{lemma}\label{lem:arbitrary_convergence}
    Let Assumptions \ref{ass:lip}, \ref{ass:minimizer} be satisfied and $d \geq 16$. Then for iterations of Algorithm \ref{alg:theor1} with $\delta \leq \cos(\beta^0)\leq 1-\delta$, $\delta \in \l(0;\frac{1}{2}\r)$ and $\varepsilon = \frac{d^{-\frac{3}{2}}  \delta M}{960},\gamma_x^t = \frac{\delta^2}{3\cdot2^{14} \cdot 5^2 L},\gamma_\mu^t = \frac{\|\mu^t\|^2}{640}$ the following bound holds:

\begin{eqnarray}
    \sum_{t=0}^{T-1} \frac{1}{T} 
    \E\l[ 
    B^t\|\nabla f(x^{t})\|^2 \r] \leq 2\frac{f(x^0)-f(x^{*})}{T\gamma_x^t},
\end{eqnarray}
where $M = \min_{t}\{\|\mu^t\|\}$ and $B^t = \min\l\{ \E[C^0] + t \cdot \frac{\delta^2}{3\cdot2^{12} \cdot 5^2};
    \cos\l(\frac{\delta}{32d} + L\gamma_x^t + \arccos(1-\delta)\r)^2 (1-e^{-1})
    \r\}$.
\end{lemma}
\begin{proof}
    We start from the result of Lemma \ref{lem:descent}. The choice of $\gamma_x^t = \frac{\delta^2}{3\cdot2^{14} \cdot 5^2 L}$ guarantees 
    \begin{eqnarray}\label{eq:1_step_lem}
    \gamma_x^t
    - 
    \frac{L {\gamma_x^t}^2}{2}
    \geq
    \frac{\gamma_x^t}{2}.
    \end{eqnarray}
    Subsequently, Theorem \ref{th:main} and Lemma \ref{lem:multiitteration_eval} justifies
    \begin{eqnarray}\label{eq:2_step_lem}
        \E\l[C^{t}|\mathcal{F}^{t-1}\r] \geq B^t.
    \end{eqnarray}
    Substitution \eqref{eq:1_step_lem}, \eqref{eq:2_step_lem} into Lemma \ref{lem:descent} and rearranging the terms finish the proof.
\end{proof}

Notably, in this setting, substituting $\E\l[ C^{t}\r | \mathcal{F}^{t-1}]$ behaviour described in Lemma \ref{lem:multiitteration_eval} yields a non-standard convergence criterion, namely convergence in terms of a \textit{weighted} average of squared gradient norms. The weights increase with the iteration index until reaching a saturation level. Such a result may be more transparent for practical scenarios such as memory-efficient LLM fine-tuning, where the algorithm outputs the final iterate rather than an averaged solution.

We also note that Lemma \ref{lem:arbitrary_convergence} relies on a uniform-in-iteration lower bound of the form $\cos(\beta^t) > \delta$, which results in a fixed step size $\gamma_x^t$ and a constant additive increase of $\E\l[ C^{t}\r | \mathcal{F}^{t-1}]$ per iteration. However, as the policy improves, $\cos(\beta^t)$ increases, which simultaneously accelerates the growth of $\E\l[ C^{t}\r | \mathcal{F}^{t-1}]$ and allows for larger admissible step sizes $\gamma_x^t$. Consequently, the uniform bound becomes loose after a certain number of iterations and does not fully capture these acceleration effects in the theoretical rate. Along similar lines, bounds of the type used in the proof of Lemma \ref{lem:ZO_SGD} yield average values on the order of $\frac{1}{d}$ both for $\delta^2$ (consequently for $\gamma_x^t$) and $\E[C^0]$.   
It introduces a factor $d^2$ in the rate of Lemma \ref{lem:arbitrary_convergence} during the initial phase and results in growth of $\E\l[ C^{t}\r | \mathcal{F}^{t-1}]$ on the order of $\frac{1}{d}$ per iteration.

\subsection{Initialization leveraging gradient direction at the initial point} \label{sec:no_prior_init_1}

\begin{lemmarepeat}{lem:parallel}
    Let Assumptions \ref{ass:lip}, \ref{ass:minimizer} be satisfied, $d \geq 16$,  and $\mu^0 ~\|~ \nabla f(x^0)$. Then for iterations of Algorithm \ref{alg:theor1} with $\delta \leq \cos(\beta^0)\leq 1-\delta$, $\delta \in \l(0;\frac{1}{2}\r)$ and $\varepsilon = \frac{d^{-\frac{3}{2}}  \delta M}{960}, M = \min_{t}\{\|\mu^t\|\},\gamma_x^t = \frac{\delta^2}{3\cdot2^{14} \cdot 5^2 L},\gamma_\mu^t = \frac{\|\mu^t\|^2}{640}$ the following bound holds:
    \begin{eqnarray}
        \E[\| \nabla f(\overline{x}^T)\|^2] &\leq& 10 \frac{f(x^0) - f(x^*)}{T\gamma_x}. \notag
    \end{eqnarray}
\end{lemmarepeat}
\begin{proof}
    We start from noticing that \eqref{eq:1_step_lem} remains true. Additionally, given $\mu^0 ~\|~ \nabla f(x^0)$ in the setting of Lemma \ref{lem:multiitteration_eval} Algorithm \ref{alg:theor1} is already in the regime where $\E\l[ C^{t}\r | \mathcal{F}^{t-1}]$ fluctuates in a neighbourhood of $1$ and attains values no smaller than $\cos\l(\frac{\delta}{32d} + L\gamma_x^t + \arccos(1-\delta)\r)^2 (1-e^{-1})$. We choose $\delta = \frac{1}{4}$ which yields
    \begin{eqnarray}\label{eq:parallel_lem_C}
        \E\l[ C^{t}\r | \mathcal{F}^{t-1}] \geq \frac{1}{5}, \text{ for all }t.
    \end{eqnarray}

    Substituting \eqref{eq:1_step_lem} with \eqref{eq:parallel_lem_C} into the Lemma \ref{lem:descent} and rearranging the terms we derive the desired result.
\end{proof}
\end{appendixpart}
\end{document}